\documentclass{article}

 \usepackage[main, final]{neurips_2026}

\usepackage[utf8]{inputenc} 
\usepackage[T1]{fontenc}    
\usepackage{url}            
\usepackage{booktabs}       
\usepackage{amsfonts}       
\usepackage{nicefrac}       
\usepackage{microtype}      
\usepackage{xcolor}         

\usepackage{graphicx}
\usepackage{amsmath}
\usepackage{amsthm}
\usepackage{bbm}
\usepackage{algorithm}
\usepackage{algorithmic}
\usepackage{dsfont}
\usepackage{listings}

\usepackage[most]{tcolorbox}
\usepackage{enumitem}

\usepackage[hidelinks]{hyperref}

\newtheorem{theorem}{Theorem}[section]
\newtheorem{proposition}[theorem]{Proposition}

\newtheorem{corollary}[theorem]{Corollary}
\newtheorem{definition}[theorem]{Definition}
\newtheorem{assumption}[theorem]{Assumption}
\newtheorem{remark}[theorem]{Remark}

\setcitestyle{authoryear,round}

\title{Ergodic Risk Measures:\\Towards a Risk-Aware Foundation for\\Continual Reinforcement Learning}

\author{%
  Juan Sebastian Rojas \\
  University of Toronto\\
  \texttt{juan.rojas@mail.utoronto.ca} \\
  \And
  Chi-Guhn Lee \\
  University of Toronto\\
  \texttt{cglee@mie.utoronto.ca} \\
}

\begin{document}

\maketitle

\begin{abstract}
Continual reinforcement learning (continual RL) seeks to formalize the notions of lifelong learning and endless adaptation in RL. In particular, the aim of continual RL is to develop RL agents that can maintain a careful balance between retaining useful information and adapting to new situations. To date, continual RL has been explored almost exclusively through the lens of risk-neutral decision-making, in which the agent aims to optimize the expected long-run performance. In this work, we present the first formal theoretical treatment of continual RL through the lens of \emph{risk-aware} decision-making, in which the behaviour of the agent is directed towards optimizing a measure of long-run performance beyond the mean. In particular, we show that the classical theory of risk measures, widely used as a theoretical foundation in non-continual risk-aware RL, is, in its current form, incompatible with continual learning. Then, building on this insight, we extend risk measure theory into the continual setting by introducing a new class of \emph{ergodic risk measures}, and showing that it is compatible with continual learning. Finally, we provide a case study of continual risk-aware learning, along with empirical results, which show the intuitive appeal of ergodic risk measures in continual settings.
\end{abstract}

\section{Introduction}
Reinforcement learning (RL) \citep{Sutton2018-eh} has enjoyed success over the years when tackling certain problems of interest in various domains, ranging from video games to robotics. However, the RL agents behind these successes are typically trained in static environments, and are evaluated on a single task for a finite period of time. By contrast, RL agents deployed in the real world may be required to operate indefinitely in scenarios where the environment and/or task changes over time. This discrepancy has motivated the study of continual reinforcement learning (continual RL) \citep{Ring1994-ww, Khetarpal2022-br, Abel2023-xe, Kumar2025-cj}, which formalizes the challenges of lifelong learning and endless adaptation in RL. At the heart of continual RL is the \emph{stability-plasticity dilemma}, through which an agent must learn to preserve sufficient prior knowledge, while still remaining sufficiently flexible to adapt to new streams of experience.

To date, the notion of continual RL has been studied almost exclusively under a \emph{risk-neutral} lens, such that the behaviour of the agent is directed towards optimizing some expectation-based measure of long-run performance (e.g. the expected sum of discounted rewards or the average per-step reward). In this work, we present the first formal theoretical exploration of continual RL through the lens of \emph{risk-aware} decision-making, in which the behaviour of the agent is directed towards optimizing a measure of long-run performance beyond the mean. Through this exploration, we seek to provide a theoretically grounded formalism through which we can frame the notion of risk-awareness in a continual setting. Namely, the first half of our exploration is guided by the following questions:
\vspace{-4pt}
\begin{itemize}[label=\(\triangleright\)]
    \setlength{\itemindent}{-19pt}
    \item \emph{What does it mean for an agent to be risk-aware in a continual setting?}\\ 
    \item \emph{What conditions are needed to enable an agent to be risk-aware in a continual setting?}\\ 
    \item \emph{What are the implications of risk-awareness as it relates to the stability-plasticity dilemma?} 
\end{itemize}

To answer these questions, we first propose a risk-aware generalization of the continual RL framework proposed in \citet{Abel2025-ez}, where we formalize continual risk-aware RL as an ongoing exchange of actions and observations between an agent and an environment, such that the agent receives observations from the environment, processes those observations into \emph{risk-aware observations} based on its \emph{risk attitude} (or \emph{risk tolerance}), then emits actions back to the environment based on the risk-aware observations. Then, we leverage this risk-aware framework to propose a set of axioms that establish the conditions needed to enable an agent to be risk-aware in continual settings.

\begin{figure}[htbp]
\centerline{\includegraphics[scale=0.46]{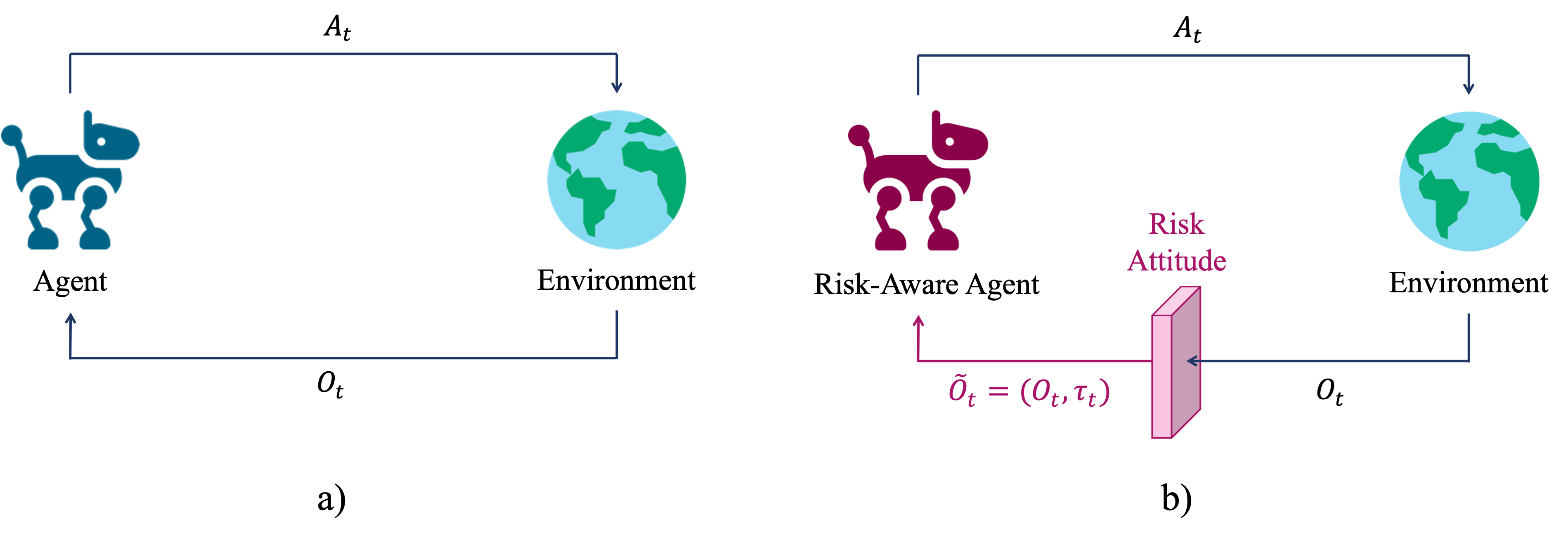}}
\caption{A comparison of: \textbf{a)} the continual RL framework proposed in \citet{Abel2025-ez}, and \textbf{b)} the risk-aware generalization of the framework proposed in this work.}
\label{fig_overview}
\end{figure}

The second half of our exploration is then aimed at examining existing frameworks and methodologies used for non-continual risk-aware RL to see how compatible they are with these axioms. To this end, we examine the classical theory of \emph{risk measures} (e.g. see Chapter 6 of \citet{Shapiro2009-pp}), which has served as a theoretical foundation for risk-aware decision-making in non-continual RL, and find that, in its current form, it is incompatible with the axioms, and hence, continual learning. Then, building on this insight, we extend risk measure theory into the continual setting by proposing a new class of \emph{ergodic risk measures}, and showing that it is compatible with continual learning. 

Finally, using the well-known average-reward Markov decision process (MDP) formulation \citep{Puterman1994-dq} as a basis, we provide a case study, along with numerical results, which show the intuitive appeal of ergodic risk measures in continual settings. Altogether, these contributions provide, to the best of our knowledge, the first formal theoretical treatment of risk-aware decision-making in a continual (i.e., lifelong) learning setting.

\section{Related Work}

\subsection{Continual Reinforcement Learning}
\label{related_continual}
The notions of lifelong learning and endless adaptation in the context of RL have long been studied under different names and perspectives. In recent years, several works (e.g. \citet{Khetarpal2022-br, Abel2023-xe, Kumar2025-cj}) have attempted to unify and frame these diverse sets of works as instances of continual RL. Some of the more common types of RL-related works that can be interpreted as being instances of continual RL include the study of the loss of plasticity in deep RL agents (e.g. \citet{Abbas2023-iw, Dohare2024-vb}), transfer learning (e.g. \citet{Abel2018-nr, Gimelfarb2021-hg}), and decision-making in non-stationary environments (e.g. \citet{Dick2014-fy, Luketina2022-gs}). In this work, we place a great emphasis on the notion of \emph{plasticity}, which has been studied extensively in prior works such as \citet{Raghavan2021-yp}, \citet{Chen2023-ph}, and \citet{Abel2025-ez}. In our proposed framework, we build on the definition of plasticity (and subsequent continual RL framework) proposed in \citet{Abel2025-ez}. 

\subsection{Risk-Aware Reinforcement Learning}
\label{related_risk}
The notion of risk-aware learning and decision-making in the context of RL has been studied under various theoretical frameworks, from the well-established expected utility framework \citep{Howard1972-zv}, to the more contemporary framework of risk measures (e.g. Chapter 6 of \citet{Shapiro2009-pp}). In this work, we focus on the latter framework, which originated in the finance literature (e.g. \citet{Rockafellar2000-xu}), but has since been widely integrated into RL-based works (e.g. \citet{Bauerle2011-yi}). Traditionally, non-continual risk-aware RL works have aimed to optimize either a \emph{static} risk measure (e.g. \citet{Mead2025-th}), or a \emph{nested (dynamic)} risk measure (e.g. \citet{Wang2025-om}). The case study presented in this work primarily focuses on the conditional value-at-risk (CVaR) risk measure \citep{Rockafellar2000-xu}, which has been studied extensively in the discounted setting (e.g. \citet{Bauerle2011-yi, Muni2026-kf}), and, to a lesser extent, in the average-reward setting (e.g. \citet{Xia2023-cq, Rojas2025-bf}).

\section{Preliminaries}
\label{prelim}

\subsection{Continual Reinforcement Learning}
\label{prelim_crl}
Continual reinforcement learning (continual RL) can be viewed as an agent-environment interaction occurring over an infinite time horizon. In particular, following the framework proposed in \citet{Abel2025-ez}, this interaction can be modeled as an ongoing exchange of signals between an agent and an environment, where, at each discrete time step, \(t\), the agent emits an action from an action-space, \(A_t \in \mathcal{A}\), and the environment emits an observation from an observation-space, \(O_t \in \mathcal{O}\) (see Figure \ref{fig_overview}a)). Accordingly, for a given agent and environment, we can define their interaction as a pair of sequences of random variables, \(O_{a:b} \doteq (O_a, O_{a+1}, \ldots, O_{b-1}, O_b)\) and \(A_{c:d} \doteq (A_c, A_{c+1}, \ldots,A_{d-1}, A_d)\), where \(a, b, c,\) and \(d\) denote discrete time indices. 

We next discuss the notion of \emph{plasticity}. In particular, one of the defining aspects of continual RL is the stability-plasticity dilemma, through which an agent must carefully balance the degree to which newly acquired information affects its behaviour relative to previously learned knowledge. More formally, the plasticity, \(\mathfrak{P}\), of a given agent, \(\lambda\), relative to a set of environments, \(\mathcal{E}\), with respect to the time intervals \([a:b] \doteq [a, a+1, \ldots, b-1, b]\), \(1 \leq a \leq b\), and \([c:d] \doteq [c, c+1, \ldots, d-1, d]\), \(1 \leq c \leq d\), can be defined as follows \citep{Abel2025-ez}:
\vspace{3pt}
\begin{equation}
\label{eq_plasticity}
    \mathfrak{P}_{\substack{a:b \\ c:d}}(\lambda, \mathcal{E}) \doteq \max_{e \in \mathcal{E}}\mathbb{I}(O_{a:b} \to A_{c:d}),
\end{equation}
where \(\mathbb{I}(X_{a:b} \to Y_{c:d}) \doteq \sum_{i=\max(a,c)}^{d}\mathbb{I}(X_{a:\min(b,i)}; Y_i \mid X_{1:a-1}, Y_{1:i-1})\) denotes the \emph{generalized directed information} \citep{Abel2025-ez} between two sequences of random variables, \(X_{a:b}\) and \(Y_{c:d}\), and \(\mathbb{I}(X;Y)\) denotes the mutual information between random variables \(X\) and \(Y\). Here, we utilize the same formal definitions for the agent and environment as \citet{Abel2025-ez}. These definitions are also provided in Appendix \ref{appendix_ergodicity}. In words, Equation \eqref{eq_plasticity} interprets plasticity as a measure of how much a sequence of observations from a set of environments can influence a sequence of the agent's actions. Accordingly, under this framework, the stability-plasticity dilemma can be interpreted as the dilemma associated with determining the right amount of plasticity, \(\mathfrak{P}\).

Importantly, the definition for plasticity in Equation \eqref{eq_plasticity} applies to both continual and non-continual learning settings. Accordingly, given our focus on continual learning, we will restrict our attention to a set of environments that give rise to meaningful continual learning scenarios. In particular, as discussed in Section 5.2 of \citet{Abel2025-ez}, in the context of continual learning, one can think of the set of environments, \(\mathcal{E}\), as being comprised of a base environment, \(e\), along with a set of \emph{interventions}, \(\mathcal{K}\), such that each intervention, \(k \in \mathcal{K}\), creates a new environment, \(e' = k(e)\), that produces a (potentially) different sequence of observations than the base environment. Accordingly, assuming a meaningful set of interventions (that yield meaningful continual learning scenarios; see Section 5.2 of \citet{Abel2025-ez}), we can further interpret the agent's plasticity, in the context of continual learning, as its \emph{capacity} to adapt to changes (i.e., interventions) in the environment. For the remainder of this work, we will implicitly assume that our agent is operating with respect to such a set of environments. This is formalized as Assumption \ref{assumption_env} below:
\begin{assumption}
\label{assumption_env}
The agent operates with respect to a set of environments, \(\mathcal{E}\), consisting of a base environment, \(e\), and a set of intervention-induced environments, such that \(\mathcal{E} \doteq \{e\} \cup \{k(e) \mid k \in \mathcal{K}\}\), where \(\mathcal{K}\) denotes a set of interventions that induce meaningful continual learning scenarios.
\end{assumption}

\subsection{Risk Measures}
\label{prelim_rm}
Let \((\Omega, \mathcal{F}, \mathbb{P})\) denote a probability space, and let \(\mathcal{X}\) denote a space of random variables of the form \(X: \Omega \to \mathbb{R}\). A \emph{risk measure} (e.g. see Chapter 6 of \citet{Shapiro2009-pp}) is a functional, 
\(\rho: \mathcal{X} \times \mathcal{T} \to \mathbb{R}\), that assigns to each random variable, \(X \in \mathcal{X}\), a real value representing the degree of risk associated with \(X\) with respect to some \emph{risk attitude} or \emph{risk tolerance}, \(\tau \in \mathcal{T}\), where \(\mathcal{T}\) denotes a set that captures the agent's risk attitude. In other words, a risk measure is a mapping that quantifies the risk associated with a random variable based on a given risk attitude. 

The precise interpretation of what ``risk'' means depends on the specific risk measure used; different risk measures capture different risk-related aspects of a random variable. In essence, one can think of a risk measure as any functional that captures distributional characteristics of a random variable, typically beyond just its mean. However, an emphasis is usually placed on deriving risk measures that satisfy certain mathematical properties that can be meaningful in risk-based decision-making contexts. In Appendix \ref{appendix_risk_measures}, we provide formal definitions of the various classes of risk measures used in the context of RL, where each (non-mutually exclusive) class of risk measures can be thought of as satisfying a specific set of mathematical properties. We note that it is commonplace to write \(\rho(X)\) or \(\rho_\tau(X)\) as shorthand for \(\rho(X, \tau)\).

\subsection{Non-Continual Risk-Aware Reinforcement Learning}
\label{prelim_risk_rl}
Traditionally, non-continual risk-aware RL works have aimed to optimize either a \emph{static} risk measure (e.g. \citet{Mead2025-th}), or a \emph{nested} risk measure (e.g. \citet{Wang2025-om}). We now provide informal definitions of these two classes of risk measures, as well as an informal definition pertaining to the notion of \emph{time consistency} as it relates to risk measures (see Appendix \ref{appendix_risk_measures} for formal definitions):

\textbf{Static Risk Measures (Definition \ref{defn_static})}: A \emph{static} risk measure, \(\rho_N\), evaluates risk at a fixed point in time, \(N\). In RL-based contexts, static risk measures can be useful for quantifying the risk of the return at the end of an episode. The primary appeal of static risk measures is their interpretability.

\textbf{Nested Risk Measures (Definition  \ref{defn_nested})}:
A \emph{nested} risk measure is a type of \emph{dynamic} risk measure (i.e., a sequence of one-step risk measures, \(\{\rho_t\}_{t=0}^N\); see Definition \ref{defn_dynamic}), that is constructed recursively from the one-step risk measures, such that the risk at time \(n\), \(\rho_n\), can be computed as \(\rho_{n}(X) \doteq \rho_n \big( \rho_{n+1}(\cdots \rho_{N-1}(X) \cdots ) \big)\). Nested risk measures are useful because they ensure \emph{time consistency} (described below). In RL-based contexts, nested risk measures are also useful because they can induce Bellman-like recursions with appealing dynamic programming-like properties. However, nested risk measures are typically hard to interpret. We note that typically in the RL literature, the terms ``dynamic risk measure'' and ``nested risk measure'' are used interchangeably; however, formally speaking, dynamic risk measures need not be time-consistent, nor have the nested structure.

\textbf{Time-Consistent Risk Measures (Definition \ref{defn_time_consistent})}: A dynamic risk measure, \(\{\rho_t(X)\}_{t=0}^N\), is said to be \emph{time-consistent} if, for all \(X, X' \in \mathcal{X}\) and all \(t < N\), we have that \(\rho_{t+1}(X) \leq \rho_{t+1}(X') \implies \rho_t(X) \leq \rho_t(X')\). Time-consistent risk measures can be appealing because they ensure that if one future outcome is deemed less risky than another at some time step, \(t\), then that same outcome is not deemed more risky than the other at any other time step. One way to think about time consistency in RL-based contexts is to ask the question: \emph{can the agent change its mind about how relatively risky something is based on new information?} If the answer is \emph{yes}, then there exists a lack of time consistency.

\section{Continual Risk-Aware Reinforcement Learning}
\label{continual_risk_aware}
What does it mean for a lifelong agent to be risk-aware? In this section, we seek to answer this question by formalizing the notion of risk-awareness in the continual setting. In particular, we motivate and propose a risk-aware generalization of the continual RL framework proposed in \citet{Abel2025-ez}, such that we formalize continual risk-aware RL as an ongoing exchange of \emph{risk-aware observations} and actions between an agent and an environment (see Figure \ref{fig_overview}). To this end, we begin by formalizing the notion of a risk-aware observation:

\begin{definition}[Risk-Aware Observation]
\label{defn_ra_obvs}
A \emph{risk-aware observation} at time \(t\), \(\tilde{O}_t \in \tilde{\mathcal{O}}\), is the tuple \((O_t, \tau_{_t})\), where \(O_t \in \mathcal{O}\) denotes an observation from the environment, \(\tau_{_t} \in \mathcal{T}\) denotes the agent's risk attitude, and \(\tilde{\mathcal{O}}\) denotes the space of risk-aware observations.
\end{definition}
In other words, a risk-aware observation can be interpreted as an observation from the environment that has been ``augmented'' based on the agent's risk attitude. The primary appeal of such a formulation is that it admits a natural connection to risk measures (see Section \ref{prelim_rm}). In particular, we can define an agent's \emph{risk assessment} over a sequence of observations, \(O_{a:b}\), as follows:

\vspace{2pt}

\begin{definition}[Risk Assessment]
\label{defn_risk_assessment}
An agent's \emph{risk assessment} over a sequence of observations collected during the interval \([a:b]\), \(U_{a:b} \doteq f(\tilde{O}_{a:b}) \doteq f(O_{a}, O_{a+1}, \ldots, O_{b-1}, O_{b}, \tau_{a}, \tau_{a+1}, \ldots, \allowbreak \tau_{b-1}, \tau_{b})\), is the output of the mapping \(f: \tilde{\mathcal{O}}^* \to \mathcal{U}\), where \(\tilde{\mathcal{O}}^*\) denotes the space of risk-aware observation sequences and \(\mathcal{U}\) denotes the space of risk assessments.
\end{definition}
In essence, a risk assessment allows the agent to capture the risk associated with a sequence of observations collected during a given time interval, such that the risk-aware observations act as a ``bridge'' between the observations emitted from the environment and the output of the risk assessment. Importantly, if the output of the risk assessment, \(U_{a:b}\), is a scalar (i.e., \(\mathcal{U} = \mathbb{R}\)), then, by definition, the output of the risk assessment corresponds to a risk measure over \(O_{a:b}\). For example, \(U_{a:b}\) could correspond to a risk measure over a sum of observations collected during the time interval, or an average of risk measures over (one-step) observations collected during the time interval.

We can also use Definition \ref{defn_ra_obvs} to generalize the definition of plasticity from \citet{Abel2025-ez} (i.e., Equation \eqref{eq_plasticity}) for a risk-aware agent as follows:
\begin{equation}
\label{eq_plasticity_risk}
    \mathfrak{R}_{\substack{a:b \\ c:d}}(\lambda, \tilde{\mathcal{E}}) \doteq \max_{\tilde{e} \in \tilde{\mathcal{E}}}\mathbb{I}(\tilde{O}_{a:b} \to A_{c:d}),
\end{equation}
where \(\tilde{O}_{a:b}\) denotes a sequence of risk-aware observations over the interval \([a:b]\), and \(\mathfrak{R}\) denotes the \emph{risk-aware plasticity}. We use \(\tilde{\mathcal{E}}\) to denote the set of pairs \(\tilde{e} \doteq (e, \tau)\) of an environment, \(e \in \mathcal{E}\), and a risk attitude, \(\tau \in \mathcal{T}\), which together determine the sequence of risk-aware observations, \(\tilde{O}_{a:b}\). As in Section \ref{prelim_crl}, we implicitly assume that Assumption \ref{assumption_env} applies to the set \(\tilde{\mathcal{E}}\). In words, the risk-aware plasticity \eqref{eq_plasticity_risk} captures the degree to which the agent's risk assessments and subsequent behaviour can be influenced by changes in its risk attitude as well as by new observations from the environment.

Altogether, Definitions \ref{defn_ra_obvs}, \ref{defn_risk_assessment}, and Equation \eqref{eq_plasticity_risk} constitute a risk-aware generalization of the continual RL framework proposed in \citet{Abel2025-ez}, such that continual risk-aware RL can be interpreted as an ongoing exchange of risk-aware observations and actions between an agent and an environment that occurs over an infinite time horizon, in which the agent's behaviour is informed and guided by risk assessments performed over sequences of observations from the environment.

\subsection{Axioms for Risk-Awareness in the Continual Setting} 
Having formalized the notion of risk-awareness in the context of continual RL, we now seek to establish the minimum requirements needed for an agent to be risk-aware in a continual learning setting. To this end, we argue that any limitations related to the agent's risk assessments should be direct consequences of the limitations of the agent itself. In particular, we focus on two well-argued limitations (or requirements) of continual RL agents (e.g. see \citet{Abel2023-xe}, \citet{Kumar2025-cj}, \citet{Abel2025-ez}): 1) any continual RL agent should be able to operate under finite memory constraints, and 2) any continual RL agent should have a non-zero level of plasticity. As such, given these limitations of the agent, we propose the following axioms as the minimum requirements needed for risk-awareness in the continual setting:
\begin{tcolorbox}[
    colback=gray!5,
    colframe=black!60,
    arc=3pt,
    boxrule=0.8pt,
    left=10pt, right=10pt, top=6pt, bottom=6pt]
\begin{proposition}
\label{proposition_axioms}
Under Assumption \ref{assumption_env}, the following two axioms are necessary and sufficient conditions for risk-awareness in the continual setting:
\begin{enumerate}[leftmargin=25pt,labelsep=6pt] 
    \vspace{2pt}
    \item \textbf{Feasibility Axiom}: The agent's risk assessments must be accurately computable based on the agent's finite memory. That is, the risk assessments must be computable based on observations collected over a finite time interval \([a:b]\) with respect to the agent's risk attitude during that same time interval.
    \vspace{5pt}
    \item \textbf{Plasticity Axiom}: \(\mathbb{P}(\mathfrak{R} > 0 \text{ infinitely often}) = 1\). In words, this axiom states that the agent's risk assessments and subsequent behaviour should be influenced by changes in its risk attitude as well as by new observations from the environment infinitely often.
\end{enumerate}
\end{proposition}
\end{tcolorbox}

\subsection{Existing Risk Measures and Continual Learning}
\label{existing}
Having established the minimum requirements needed for risk-awareness in continual learning settings, we now turn our attention towards answering the following question: do the existing classes of risk measures used in non-continual risk-aware RL (i.e., static and nested; see Section \ref{prelim_risk_rl}) satisfy the requirements outlined in Proposition \ref{proposition_axioms}? In other words, if we assume that the output of an agent's risk assessment is a scalar (this is formalized as Assumption \ref{assumption_scalar} below), are the existing classes of risk measures compatible with continual learning? As it turns out, static and nested risk measures \emph{do not} satisfy the requirements of Proposition \ref{proposition_axioms}. This is formally shown in Propositions \ref{proposition_static_not_plastic} and \ref{proposition_nested_not_plastic} below:

\vspace{4pt}

\begin{assumption}
\label{assumption_scalar}
The output of an agent's risk assessment (see Definition \ref{defn_risk_assessment}) is a scalar.
\end{assumption}

\vspace{2pt}

\begin{proposition}
\label{proposition_static_not_plastic}
Let \(\rho_N\) denote a static risk measure (as per Definition \ref{defn_static}) defined for some time step, \(N\). The static risk measure, \(\rho_N\), does not satisfy both axioms of Proposition \ref{proposition_axioms}.
\end{proposition}

\begin{proof}
As \(N \to \infty\), the Feasibility axiom is violated since the computation of the risk would require the agent to have infinite memory. If \(N < \infty\), then any observation that occurs after time step \(N\) would not influence the risk assessment, thereby violating the Plasticity axiom. 
\end{proof}

\vspace{1pt}

\begin{proposition}
\label{proposition_nested_not_plastic}
Let \(\{\rho_t\}_{t=0}^\infty\) denote a nested risk measure and let \(\rho_{n}\) denote the assessed risk at time step \(n\), such that its value can be computed as: \(\rho_{n}(X) \doteq \rho_n \big( \rho_{n+1}(\cdots) \big)\) (see Definition \ref{defn_nested}). The nested risk measure, \(\{\rho_t\}_{t=0}^\infty\), does not satisfy both axioms of Proposition \ref{proposition_axioms}.
\end{proposition}
\begin{proof}
Consider the risk assessment at time \(n\), \(\rho_{n}(X) \doteq \rho_n \big( \rho_{n+1}(\cdots ) \big)\). By definition, the validity of this risk calculation is contingent on time consistency holding for all \(t \geq n\). However, this implies that future observations cannot influence how risky \(X\) is relative to other outcomes. Hence, either time consistency is satisfied, which violates the Plasticity axiom, or time consistency is violated, which invalidates the risk calculation at time \(n\), thereby violating the Feasibility axiom.
\end{proof}

\subsection{Ergodic Risk Measures for Continual Learning}
\label{ergodic}
In the previous section, we showed that static and nested risk measures are not compatible with continual learning. Based on this insight, we now propose a new class of \emph{ergodic risk measures}, and show that it is compatible with Proposition \ref{proposition_axioms}, and hence, continual learning:
\begin{tcolorbox}[
    colback=gray!5,
    colframe=black!60,
    arc=3pt,
    boxrule=0.8pt,
    left=10pt, right=10pt, top=6pt, bottom=6pt]
\begin{definition}[Ergodic Risk Measure]  
\label{defn_ergodic} 
Let \(\{\rho_t\}_{t=0}^\infty\) denote a sequence of conditional risk measures (i.e., a dynamic risk measure) defined on a probability space, \((\Omega, \mathcal{F}, \mathbb{P})\), with filtration \((\mathcal{F}_t)_{t=0}^\infty\), where \(\mathcal{F}_t \subseteq \mathcal{F}\) denotes the \(\sigma\)-algebra representing the information available up until time step \(t \geq 0\), such that \(\rho_t : L^\infty(\mathcal{F}_\infty) \to L^\infty(\mathcal{F}_t)\), where \(L^\infty(\mathcal{F}_u)\) denotes the space of essentially bounded, \(\mathcal{F}_u\)-measurable random variables.\\ 

Suppose that for any time step, \(t > 0\), there exists a finite time interval, \([n:t]\), such that for all \(X \in L^\infty(\mathcal{F}_\infty)\), and some \(\hat{\rho}_t(X) \in L^\infty(\mathcal{G}_{n:t})\), where \(\mathcal{G}_{n:t}\) denotes the \(\sigma\)-algebra generated by the information from time step \(n\) up to time step \(t\), we have that: 
\begin{equation}
\label{eq_ergodicity}
\rho_t(X) - \hat{\rho}_t(X) \approx 0.
\end{equation}
That is, that the risk assessment at any given time step, \(t\), can be accurately computed based on a preceding, finite subset of the history, such that the risk assessment effectively ceases to depend on the history that occurs prior to the interval \([n:t]\). Furthermore, suppose that \(\{\rho_t\}_{t=0}^\infty\) is designed in such a way that time consistency, as defined in Definition \ref{defn_time_consistent}, is not strictly enforced. Then, we call \(\{\rho_t\}_{t=0}^\infty\) an \emph{ergodic risk measure}.
\end{definition}
\end{tcolorbox}

In essence, an ergodic risk measure is a non-nested dynamic risk measure, such that its one-step risk assessments, \(\rho_t\), can be \emph{accurately} computed using only observations collected over a rolling finite time interval, \([n:t]\). We note that the size of this interval can vary between different time steps.

We now show that ergodic risk measures are compatible with both axioms of Proposition \ref{proposition_axioms}:

\vspace{4pt}

\begin{proposition}
Let \(\{\rho_t\}_{t=0}^\infty\) denote an ergodic risk measure as defined in Definition \ref{defn_ergodic}. The ergodic risk measure, \(\{\rho_t\}_{t=0}^\infty\), is compatible with both axioms of Proposition \ref{proposition_axioms}.
\end{proposition}
\begin{proof}
Since an ergodic risk measure is a dynamic risk measure such that time consistency (as defined in Definition \ref{defn_time_consistent}) is not strictly enforced, it is compatible with the Plasticity axiom since the risk assessments and subsequent behaviour can be influenced by new observations and any changes in the agent's risk attitude infinitely often. Similarly, by definition, the risk at any time \(t\) can be computed using observations collected over a finite time interval with respect to the agent's risk attitude during that same time interval, thereby making it compatible with the Feasibility axiom.
\end{proof}

\subsection{Ergodic Risk Measures as Reinforcement Learning Objectives}
\label{section_average_reward}
Having established that ergodic risk measures are compatible with continual learning, we now motivate a corresponding RL objective that can be used to perform risk-aware learning and decision-making in the continual setting. To this end, we will utilize the average-reward MDP formulation \citep{Puterman1994-dq} to realize the agent-environment interaction described in Section \ref{prelim_crl}. Our choice of the average-reward formulation in this work follows prior work, such as \citet{Riemer2022-no}, \citet{Sharma2022-ua}, and \citet{Kumar2025-cj}, which argue that the average-reward formulation's emphasis on long-term performance is a natural fit for continual learning settings (also see Appendix \ref{appendix_disc}).

More formally, consider a finite average-reward MDP, \(\mathcal{M} \doteq \langle \mathcal{S}, \mathcal{A}, \mathcal{R}, p \rangle\), where \(\mathcal{S}\) is a finite set of states, \(\mathcal{A}\) is a finite set of actions, \(\mathcal{R} \subset \mathbb{R}\) is a bounded set of rewards, and \(p: \mathcal{S}\, \times\, \mathcal{A}\, \times\, \mathcal{R}\, \times\,  \mathcal{S} \rightarrow{} [0, 1]\) is a probabilistic transition function that describes the dynamics of the environment, such that at each discrete time step, \(t = 0, 1, 2, \ldots\), an agent chooses an action, \(A_t \in \mathcal{A}\), based on its current state, \(S_t \in \mathcal{S}\), and receives a reward, \(R_{t+1} \in \mathcal{R}\), while transitioning to a (potentially) new state, \(S_{t+1}\), such that \(p(s', r \mid s, a) = \mathbb{P}(S_{t+1} = s', R_{t+1} = r \mid S_t = s, A_t = a)\).

Under this framework, the continual RL problem can be viewed as an infinite sequence of average-reward MDPs, \(\{\mathcal{M}^k\}_{k=1}^\infty\), such that \(\mathcal{M}^k \doteq \langle \mathcal{S}_k, \mathcal{A}_k, \mathcal{R}_k, p_k \rangle\), where each \(\mathcal{M}^k\) may differ in its state-space, action-space, reward function, and/or transition dynamics based on some indexing function, \(\omega: \mathbb{N} \to \mathbb{N}\), such that \(\omega(t) = i\) indicates that at time step \(t\), the agent is in MDP \(\mathcal{M}^i\). 

The agent's goal in this framework is to construct a sequence of stationary policies, \(\{\pi^k\}_{k=1}^\infty\), that optimizes the long-run (or limiting) average-reward, \(\bar{r}\), which is defined as follows for a given stationary policy being followed at time \(t\), \(\pi_t \doteq \pi^{\omega(t)}\):
\begin{equation}
\label{eq_avg_reward}
\bar{r}^{\pi_t}(s) \doteq  \lim_{h \rightarrow{} \infty} \frac{1}{h - n_0} \sum_{n=n_0+1}^{h} \mathbb{E}[R_n \mid S_{n_0}=s, A_{n_0:n-1} \sim \pi^{\omega(n_0)}],
\end{equation}
where \(n_0 = \min\{y: \pi^{\omega(y)} = \pi_t\}\). When working with average-reward MDPs, it is commonplace to invoke \emph{ergodicity-like} assumptions (see below) so that the agent's objective \eqref{eq_avg_reward} remains well-defined and becomes independent of prior conditions (i.e., \(\bar{r}^{\pi_t}(s) = \bar{r}^{\pi_t}\)). In Appendix \ref{appendix_ergodicity}, we provide a mathematically rigorous discussion on the relationship between ergodicity and continual learning, in which we show that ergodicity does not preclude an agent's ability to be a continual learner.

\vspace{2pt}

\begin{assumption}[Unichain Assumption for Prediction]\label{assumption_unichain}
For a given MDP, \(\mathcal{M}^k\), the Markov chain induced by the policy, \(\pi^k\), is unichain. That is, the induced Markov chain consists of a single recurrent class and a potentially empty set of transient states.
\end{assumption}

\begin{assumption}[Communicating Assumption for Control] \label{assumption_communicating}
For a given MDP, \(\mathcal{M}^k\), the MDP has a single communicating class. That is, each state in the MDP is accessible from every other state in the MDP under some deterministic stationary policy.
\end{assumption}

For clarity, we note that the above assumptions need only apply independently for each MDP, \(\mathcal{M}^k\), rather than holistically for the entire sequence of MDPs, \(\{\mathcal{M}^k\}_{k=1}^\infty\). For example, Assumption \ref{assumption_communicating} only requires that each state in a given MDP, \(\mathcal{M}^k\), is accessible from every other state in \(\mathcal{M}^k\), but not the states of other MDPs, \(\mathcal{M}^{i \neq k}\). 

Now, to perform risk-aware decision-making in a continual setting, we can consider the risk-aware analogue to the (risk-neutral) average-reward RL objective \eqref{eq_avg_reward}:
\begin{equation}
\label{eq_risk_rl}
\bar{\rho}^{\pi_t}(s) \doteq \lim_{h \rightarrow{} \infty} \frac{1}{h - n_0} \sum_{n=n_0+1}^{h} \rho_\tau[R_n \mid S_{n_0}=s, A_{n_0:n-1} \sim \pi^{\omega(n_0)}].
\end{equation}
That is, we want to optimize some time-averaged measure of risk pertaining to the \emph{limiting per-step reward distribution} (see \citet{Rojas2026-dv}) induced when following a given policy, \(\pi_t\). However, the risk-aware objective presented in Equation \eqref{eq_risk_rl} is dependent on the initial conditions, which is not tractable in the continual setting. As such, as with the risk-neutral objective \eqref{eq_avg_reward}, we can apply an appropriate ergodicity-like assumption, such as Assumption \ref{assumption_unichain} or \ref{assumption_communicating}, such that the risk-aware objective \eqref{eq_risk_rl} becomes independent of prior conditions (i.e., \(\bar{\rho}^{\pi_t}(s) = \bar{\rho}^{\pi_t}\)). 

We will also assume that \(\rho_\tau[R_n | \cdot]\) can be estimated at every time step (see Remark \ref{remark_risk_functional}):
\begin{assumption}
\label{assumption_risk_functional}
Given an appropriate ergodicity-like assumption, such as Assumption \ref{assumption_unichain} or \ref{assumption_communicating}, the \emph{risk functional} in Equation \eqref{eq_risk_rl}, \(\rho_\tau[R_n \mid \cdot]\), can be estimated at every time step.
\end{assumption}

\begin{remark}
\label{remark_risk_functional}
We note that several common choices of \(\rho_\tau\) satisfy Assumption \ref{assumption_risk_functional}, including \(\rho_\tau[R_n | \cdot] = \mathbb{E}[R_n | \cdot]\), \(\rho_\tau[R_n | \cdot] = \text{Var}[R_n | \cdot]\) (e.g. see \citet{Prashanth2016-mr}), and \(\rho_\tau[R_n | \cdot] = \text{CVaR}_\tau[R_n | \cdot]\) (e.g. see \citet{Rojas2025-bf}).
\end{remark}

Importantly, we can show that, under such assumptions, the risk-aware objective \eqref{eq_risk_rl} corresponds to an ergodic risk measure, thereby making it compatible with continual learning:

\vspace{4pt}

\begin{tcolorbox}[
    colback=gray!5,
    colframe=black!60,
    arc=3pt,
    boxrule=0.8pt,
    left=10pt, right=10pt, top=6pt, bottom=6pt]
\begin{theorem}
\label{theorem_ergodic_1}
Given an appropriate ergodicity-like assumption, such as Assumption \ref{assumption_unichain} or \ref{assumption_communicating}, a stationary policy, \(\pi_t\), and Assumption \ref{assumption_risk_functional}, the risk-aware objective \eqref{eq_risk_rl} corresponds to an ergodic risk measure, as defined in Definition \ref{defn_ergodic}.
\end{theorem}
\end{tcolorbox}
\begin{proof}
First, we note that, by construction, the risk-aware objective \eqref{eq_risk_rl} does not strictly enforce time consistency. Next, we show that for any time step, \(t\), we can accurately compute the risk based on observations collected over a preceding, finite time interval. To this end, consider an MDP, \(\mathcal{M}^k\), through which an agent interacts with the environment for the possibly infinite time interval \([a:b]\). Under the stated assumptions, we can evaluate the risk-aware objective \eqref{eq_risk_rl} at time \(a \leq t \leq b\) as \(\bar{\rho}^{\pi_t}_t = \frac{1}{(t-a+1)}\sum_{n=a}^t\rho_{\tau}[R_n | A_{n} \sim \pi^{\omega(a)}]\). If \([a:b]\) is finite, then the desired condition is automatically satisfied. If \([a:b]\) is infinite, then, under the ergodicity-like assumption, we have, by Birkhoff's Ergodic Theorem \citep{Birkhoff1931-mx}, that \(\bar{\rho}^{\pi_t}_t\) converges to a stationary value as \(t \to \infty\). This implies that there exists a finite time step, \(j\), such that for all \(t \geq j\), \(\bar{\rho}^{\pi_t}_t\) converges to a stationary value. We therefore have two cases to check: 1) If \(t \leq j\), then, since \(j\) is finite, the risk can be computed based on observations collected over the finite time interval \([a:j]\), thereby satisfying the desired condition. 2) If \(t > j\), then the convergence of \(\bar{\rho}^{\pi_t}_t\) to a stationary value as \(t \to \infty\) implies that for any \(t > j\), we can find a finite time interval, \([n:t]\) (where \(n \geq j\)), that can be used to accurately compute the risk assessment, thereby satisfying the desired condition. This completes the proof.
\end{proof}

\begin{remark}
\label{remark_gap}
We note that ergodic risk measures are also compatible with the \textbf{non-continual} average-reward setting. That is, ergodic risk measures provide a formalism for risk-aware objectives of the form \eqref{eq_risk_rl}, which had been studied previously in non-continual settings (e.g. \citet{Xia2023-cq}, \citet{Rojas2025-bf}), but never formalized in the context of risk measure theory (see Appendix \ref{appendix_noncontinual}).
\end{remark}

\section{Case Study: CVaR as an Ergodic Risk Measure}
\label{cvar}
In this section, we present a case study in which we show how optimizing an ergodic risk measure can induce sensible risk-aware behaviour in continual learning settings. That is, we show empirically how optimizing an ergodic risk measure can enable the agent to successfully adapt in situations where either its risk attitude changes, or the observations from the environment change.

To this end, we focus on optimizing the well-known conditional value-at-risk (CVaR) risk measure \citep{Rockafellar2000-xu}. More formally, consider a random variable \(X\) with a cumulative distribution function, \(F(x) = \mathbb{P}(X \le x)\). The (left-tail) \emph{value-at-risk (VaR)} of \(X\) with parameter \(\tau \in (0, 1)\) represents the \(\tau\)-quantile of \(X\), such that \(\text{VaR}_{\tau}(X) = \sup\{x \mid F(x) \le \tau\}\). If \(F(x)\) is continuous at \(x = \text{VaR}_{\tau}(X)\), then \(\text{CVaR}_{\tau}(X) = \mathbb{E}[X \mid X \le \text{VaR}_{\tau}(X)]\). Importantly, as per the results in Section \ref{continual_risk_aware}, and given Assumptions \ref{assumption_unichain} or \ref{assumption_communicating}, the CVaR risk measure can be formulated as the following continual learning objective that corresponds to an ergodic risk measure:
\begin{equation}
\label{eq_cvar_1}
\overline{\text{CVaR}}^{\pi_t} \doteq  \lim_{h \rightarrow{} \infty} \frac{1}{h - n_0} \sum_{n=n_0+1}^{h} \text{CVaR}_\tau[R_n \mid A_{n_0:n-1} \sim \pi^{\omega(n_0)}].
\end{equation}

In this case study, we optimize the CVaR objective \eqref{eq_cvar_1} via the \emph{RED CVaR Q-learning} algorithm proposed in \citet{Rojas2025-bf} (see Remark \ref{remark_gap}) in two continual learning tasks. The full set of experiment details can be found in Appendix \ref{appendix_experiments}.

In the first task, we consider a continual variation of the \emph{red-pill blue-pill (RPBP)} task \citep{Rojas2025-bf}. More specifically, in the regular (non-continual) RPBP task, an agent, at each time step, can take either a \emph{red pill}, which takes them to the \emph{red world} state, or a \emph{blue pill}, which takes them to the \emph{blue world} state. Each state has its own per-step reward distribution, such that for a sufficiently low CVaR parameter, \(\tau\), the red world state has a reward distribution with a lower (worse) mean but a higher (better) CVaR compared to the blue world state. That is, in the regular RPBP task, for a sufficiently low CVaR parameter, \(\tau\), we would expect a risk-neutral agent to learn a policy that prefers to stay in the blue world, and a risk-averse agent to learn a policy that prefers to stay in the red world.

In the continual variation of RPBP considered in this task, the risk attitude of the agent, which is governed by the CVaR parameter, \(\tau\), changes over time from risk-neutral (\(\tau \approx 1 \)) to risk-averse (\(\tau \approx 0 \)). Here, we would expect the agent to first learn to stay in the blue world state, but change its preference to the red world state as its risk attitude changes from risk-neutral to risk-averse. More formally, this task can be viewed as a continual learning task with a changing reward function. We refer to this task as the \(\tau\)-RPBP task.

In the second task, we consider another continual variation of RPBP in which the per-step reward distributions of the states change over time, such that the agent is required to continually adapt and find the state with the better CVaR (given a fixed risk attitude, \(\tau\)). More formally, this task can be viewed as a continual learning task with a changing state-space (such that a given state is effectively replaced with a state with a different reward distribution). We refer to this task as the \(\mathcal{S}\)-RPBP task.

\begin{figure}[htbp]
\centerline{\includegraphics[scale=0.47]{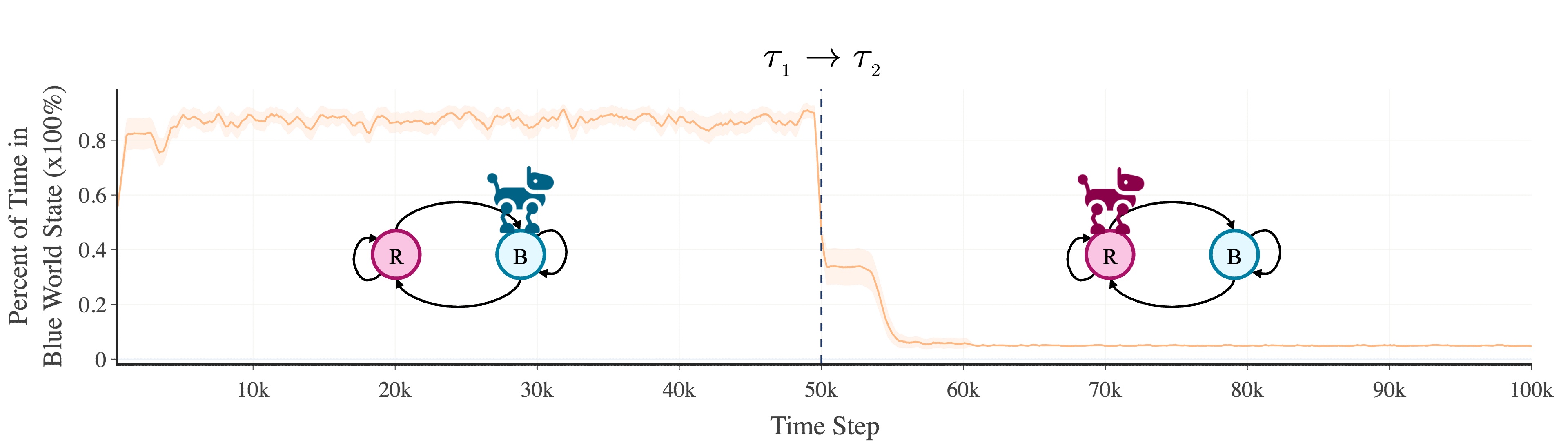}}
\caption{Rolling percent of time that the agent stays in the blue world state as learning progresses in the \(\tau\)-RPBP task. A solid line denotes the mean percent of time spent in the blue world state, and the shaded region denotes a 95\% confidence interval over 50 runs. As shown in the figure, the agent correctly learns to stay in the blue world state in the beginning, and then correctly changes its preference to the red world state once its risk attitude changes from risk-neutral to risk-averse.}
\label{fig_results_1}
\end{figure}

\vspace{6pt}

\begin{figure}[htbp]
\centerline{\includegraphics[scale=0.47]{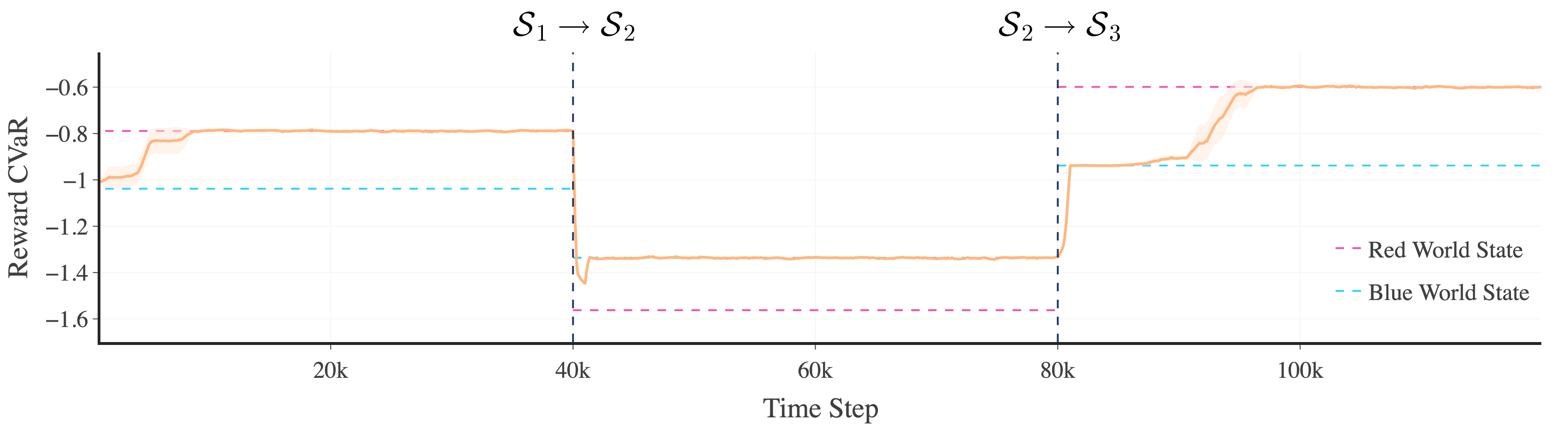}}
\caption{Rolling reward CVaR as learning progresses in the \(\mathcal{S}\)-RPBP task. A solid line denotes the mean CVaR, and the shaded region denotes a 95\% confidence interval over 10 runs. The blue and red dashed lines denote the reward CVaR of the blue and red world states, respectively. As shown in the figure, the agent is able to continually adapt and find the state with the better CVaR.}
\label{fig_results_2}
\end{figure}

\newpage

In terms of empirical results, Figures \ref{fig_results_1} and \ref{fig_results_2} show the resulting agent behaviour as learning progresses in both tasks. In particular, Figure \ref{fig_results_1} shows that in the \(\tau\)-RPBP task, the agent correctly learns to stay in the blue world state in the beginning, and then correctly changes its preference to the red world state once its risk attitude changes from risk-neutral to risk-averse. Similarly, Figure \ref{fig_results_2} shows that in the \(\mathcal{S}\)-RPBP task, the agent is able to continually adapt and find the state with the better CVaR.

\section{Discussion}
\label{discussion}
In this work, we took the first steps towards developing a risk-aware foundation for continual RL. In particular, we formalized the notion of risk-awareness in continual settings, and used the resulting framework to establish the conditions needed to enable an agent to be risk-aware in continual settings. We then examined the classical theory of risk measures, and showed that, in its current form, it is incompatible with continual learning. Then, building on this insight, we extended risk measure theory into the continual setting by introducing a new class of \emph{ergodic risk measures}, and showing that it is compatible with continual learning. Finally, we provided a CVaR-based case study, along with numerical results, which showed the intuitive appeal of ergodic risk measures in continual settings. 

More broadly, in \emph{non-continual} settings, we note that ergodic risk measures offer a new and appealing alternative for risk-aware learning and decision-making. In particular, when utilizing the average-reward formulation, ergodic risk measures offer several advantages in comparison to the static and nested risk measures that are typically used in non-continual RL settings. A mathematically rigorous discussion of these advantages is provided in Appendix \ref{appendix_noncontinual}. In summary, we show that ergodic risk measures can retain some notion of time consistency while remaining highly interpretable, thereby capturing the appeal of both static and nested risk measures. Moreover, we show that ergodic risk measures are able to circumvent some of the challenges and non-trivialities that have been historically associated with risk-aware RL objectives in non-continual settings.

We also note that this (theoretically motivated) move away from static and nested risk measures towards ergodic risk measures mirrors recent trends in the finance literature. In particular, the branch of \emph{ergodicity economics} \citep{Peters2019-hv}, which has gained popularity in recent years, argues that economic models should aim to optimize objectives (or \emph{observables}) that (provably) satisfy ergodicity-like properties. As such, the use of ergodicity-like assumptions in our approach coincides with recent trends in other (non-RL) branches of the decision-making literature. Finally, we note that although some prior works in the RL literature have questioned the compatibility between ergodicity and continual learning (e.g. \citet{Elelimy2025-mf}), we show in Appendix \ref{appendix_ergodicity}, through a mathematically rigorous discussion on the relationship between ergodicity and continual learning, that ergodicity does not preclude an agent's ability to be a continual learner.

All in all, this work represents the first formal theoretical exploration of risk-aware decision-making in a continual learning setting. Moving forward, we believe that the theoretical foundation that has been established, including the introduction of a theoretically sound risk-aware objective that is stable-yet-adaptable, will enable further progress in the development of risk-aware lifelong agents.

\section*{Acknowledgments}
We gratefully acknowledge funding from NSERC Discovery Grant \# RGPIN-2026-04777, as well as computing resources provided by the Digital Research Alliance of Canada. We thank anonymous ICLR reviewer 6n6g for their valuable feedback and suggestions during an earlier review process. Finally, we thank the anonymous NeurIPS reviewers and area chair for their useful feedback and commentary during the review process.

\bibliography{references}
\bibliographystyle{plainnat}

\newpage

\appendix

\numberwithin{equation}{section}
\numberwithin{figure}{section}
\numberwithin{theorem}{section}

\section{Risk Measures}
\label{appendix_risk_measures}
In this appendix, we provide formal definitions of the various classes of risk measures used in the context of RL, where each (non-mutually exclusive) class of risk measures can be thought of as satisfying a specific set of mathematical properties:\\

\begin{definition}[Coherent Risk Measure; adapted from \citet{Artzner1999-bs}]
\label{defn_coherent}
A risk measure, \(\rho\), is called coherent if it satisfies the following four axioms for all random variables \(X, X' \in \mathcal{X}\):
\begin{enumerate}
    \item Monotonicity: If \(X \leq X'\) almost surely, then \(\rho(X) \leq \rho(X')\).
    \item Translation Invariance: For all \(c \in \mathbb{R}\), \(\rho(X + c) = \rho(X) + c\).
    \item Positive Homogeneity: For all \(c \geq 0\), \(\rho(c X) = c \rho(X)\).
    \item Subadditivity: \(\rho(X + X') \leq \rho(X) + \rho(X')\).
\end{enumerate}
\end{definition}
Coherent risk measures are useful because they enforce a form of self-consistency in how risk is quantified and compared. In particular, monotonicity ensures that if a random variable, \(X\), always yields outcomes that are no worse than outcomes induced by another random variable, \(X'\), then \(X\) should be considered less risky than \(X'\). Translation invariance requires that adding a constant amount to \(X\) simply shifts its risk by that same amount. Positive homogeneity enforces scale-consistency, such that doubling the size of \(X\) also doubles its risk. Finally, subadditivity formalizes the idea of diversification, requiring that the risk of two combined random variables cannot exceed the sum of their individual risks. We note that the positive homogeneity and subadditivity properties also ensure that coherent risk measures are convex.\\

\begin{definition}[Static Risk Measure; adapted from \citet{Artzner1999-bs}]
\label{defn_static}
Let \((\Omega, \mathcal{F}, \mathbb{P})\) denote a probability space, and let \(\mathcal{F}_N \subseteq \mathcal{F}\) denote the \(\sigma\)-algebra representing the information available at time \(N\). Denote by \(L^\infty(\mathcal{F}_N)\) the space of essentially bounded, \(\mathcal{F}_N\)-measurable random variables. A static risk measure is a mapping, \(\rho: L^\infty(\mathcal{F}_N) \to \mathbb{R}\), that assigns to each random variable, \(X \in L^\infty(\mathcal{F}_N)\), a single real value.
\end{definition}
In essence, a static risk measure evaluates risk at a fixed point in time. In RL-based contexts, static risk measures can be useful for quantifying the risk associated with the return at the end of an episode. One of the primary appeals of static risk measures is their interpretability.\\

\begin{definition}[Conditional Risk Measure; adapted from \citet{Ruszczynski2010-sm}]
\label{defn_conditional}
Let \((\Omega, \mathcal{F}, \mathbb{P})\) denote a probability space with filtration \((\mathcal{F}_t)_{t=0}^N\), where \(\mathcal{F}_t \subseteq \mathcal{F}\) represents the information available up to time \(t\). Denote by \(L^\infty(\mathcal{F}_N)\) the space of essentially bounded, \(\mathcal{F}_N\)-measurable random variables, and by \(L^\infty(\mathcal{F}_t)\) the space of essentially bounded, \(\mathcal{F}_t\)-measurable random variables. A conditional risk measure at time \(t\) is a mapping, \(\rho_t: L^\infty(\mathcal{F}_N) \to L^\infty(\mathcal{F}_t)\), that assigns to each random variable, \(X \in L^\infty(\mathcal{F}_N)\), a conditional risk evaluation, \(\rho_t(X)\), that is \(\mathcal{F}_t\)-measurable and satisfies the following monotonicity property: if \(X \leq X'\) almost surely, then \(\rho_t(X) \leq \rho_t(X')\). 
\end{definition}
In essence, a conditional risk measure at time \(t\) is a mapping that evaluates the risk of future outcomes (e.g. at time \(N > t\)) based on the information available up to and including time \(t\). The monotonicity property ensures that if a future outcome, \(X\), always yields less loss (or more reward) than another outcome, \(X'\), then \(X\) is never assigned a higher risk than \(X'\).\\

\begin{definition}[Dynamic Risk Measure; adapted from \citet{Ruszczynski2010-sm}]
\label{defn_dynamic}
Let \((\Omega, \mathcal{F}, \mathbb{P})\) denote a probability space with filtration \((\mathcal{F}_t)_{t=0}^N\). A dynamic risk measure is a \emph{sequence} of conditional risk measures, \(\{\rho_t\}_{t=0}^N\), where each \(\rho_t: L^\infty(\mathcal{F}_N) \to L^\infty(\mathcal{F}_t)\) assigns to every random variable \(X \in L^\infty(\mathcal{F}_N)\) a conditional risk evaluation, \(\rho_t(X)\), that is \(\mathcal{F}_t\)-measurable.
\end{definition}
In essence, a dynamic risk measure is a time-indexed sequence of conditional risk measures. In RL-based contexts, dynamic risk measures can be useful for capturing the sequential nature of decision-making (i.e., that actions taken at each time step can potentially influence future outcomes, and hence, future risk evaluations).\\

\begin{definition}[Time-Consistent Risk Measure; adapted from \citet{Boda2006-tw}]
\label{defn_time_consistent}
Let \(\{\rho_t\}_{t=0}^N\) be a dynamic risk measure defined on a probability space, \((\Omega, \mathcal{F}, \mathbb{P})\), with filtration \((\mathcal{F}_t)_{t=0}^N\). The dynamic risk measure is said to be time-consistent if, for all random variables \(X, X' \in L^\infty(\mathcal{F}_N)\) and all \(t < N\), 
\begin{equation}
\label{eq_time_consistent}
\rho_{t+1}(X) \leq \rho_{t+1}(X') \implies \rho_t(X) \leq \rho_t(X').
\end{equation}
\end{definition}
Time-consistent risk measures can be appealing because they ensure that if one future outcome is deemed less risky than another at some time step, \(t\), then that same outcome is not deemed more risky than the other at any other time step. One way to think about time consistency in RL-based contexts is to ask the question: \emph{can the agent change its mind about how relatively risky something is based on new information?} If the answer is \emph{yes}, then there exists a lack of time consistency.\\

\begin{definition}[Nested Risk Measure; adapted from \citet{Ruszczynski2010-sm}]
\label{defn_nested}
A nested risk measure is a dynamic risk measure, \(\{\rho_t\}_{t=0}^N\), that is constructed recursively from one-step conditional risk measures. More formally, given one-step conditional risk measures, \(\rho_t: L^\infty(\mathcal{F}_{t+1}) \to L^\infty(\mathcal{F}_t)\), the risk at time \(n\), \(\rho_n\), can be evaluated as:
\begin{equation}
\label{eq_nested}
\rho_n(X) \doteq \rho_n \big( \rho_{n+1}(\cdots \rho_{N-1}(X) \cdots ) \big).
\end{equation}
\end{definition}
Nested risk measures are useful because they ensure time consistency, such that the risk evaluation at earlier times is consistent with future evaluations. In RL-based contexts, nested risk measures are also useful because they can induce Bellman-like recursions with appealing dynamic programming-like properties. One of the drawbacks of nested risk measures is that they are typically hard to interpret. We note that typically in the RL literature, the terms ``dynamic risk measure'' and ``nested risk measure'' are used interchangeably; however, formally speaking, dynamic risk measures need not be time-consistent, nor have the nested structure.\\ 

\begin{definition}[Markov Risk Measure; adapted from \citet{Ruszczynski2010-sm}]
\label{defn_markov}
Let \((\Omega, \mathcal{F}, \mathbb{P})\) denote a probability space, and let \(\{S_t\}_{t=0}^N\) denote a Markov process where each state, \(S_t : \Omega \to \mathcal{S}\), takes values in a measurable state-space, \((\mathcal{S}, \mathcal{B}(\mathcal{S}))\), such that \(\mathcal{F}_t \doteq \sigma(S_0,\ldots,S_t)\). Here, \(\mathcal{B}(\mathcal{S})\) denotes the Borel \(\sigma\)-algebra on \(\mathcal{S}\). A one-step conditional risk measure, \(\rho_t : L^\infty(\mathcal{F}_{t+1}) \to L^\infty(\mathcal{F}_t)\), is called a Markov risk measure if, for every \(X \in L^\infty(\mathcal{F}_{t+1})\), the risk assessment satisfies \(\rho_t(X) \in L^\infty(\sigma(S_t))\), where \(\sigma(S_t) \subseteq \sigma(S_0,\ldots,S_t) = \mathcal{F}_t\). That is, the risk assessment \(\rho_t(X)\) is \(\sigma(S_t)\)-measurable, such that it only depends on the current state, \(S_t\).
\end{definition}
Markov risk measures are useful because they enforce a one-step time dependence structure that makes them compatible with MDP-based RL solution methods. From a risk perspective, this means that the assessment of risk at each time step only depends on the information available at that time step, rather than the entire history of past information. Note that in an MDP setting (as opposed to the simpler Markov process described in Definition \ref{defn_markov}), the ``state'' can be characterized as a state-action pair. That is, \(\rho_t(X) \in L^\infty(\sigma(S_t, A_t))\) for some \(A_t\) in a measurable action-space, \(\mathcal{A}\).


\newpage

\section{Numerical Experiments}
\label{appendix_experiments}
This appendix contains details regarding the numerical experiments performed as part of this work. The overall aim of the experiments was to show how optimizing an ergodic risk measure can induce sensible risk-aware behaviour in continual learning settings. That is, we aimed to show empirically how optimizing an ergodic risk measure can enable the agent to successfully adapt in situations where either its risk attitude changes, or the observations from the environment change.

To this end, we focused on optimizing the well-known conditional value-at-risk (CVaR) risk measure \citep{Rockafellar2000-xu}. More formally, consider a random variable, \(X\), with a finite mean on a probability space, \((\Omega, \mathcal{F}, \mathbb{P})\), and with a cumulative distribution function, \(F(x) = \mathbb{P}(X \le x)\). The (left-tail) \emph{value-at-risk (VaR)} of \(X\) with parameter \(\tau \in (0, 1)\) represents the \(\tau\)-quantile of \(X\), such that \(\text{VaR}_{\tau}(X) = \sup\{x \mid F(x) \le \tau\}\). When \(F(x)\) is continuous at \(x = \text{VaR}_{\tau}(X)\), \(\text{CVaR}_{\tau}(X)\) can be interpreted as the expected value of \(X\) conditioned on \(X\) being less than or equal to \(\text{VaR}_{\tau}(X)\), such that \(\text{CVaR}_{\tau}(X) = \mathbb{E}[X \mid X \le \text{VaR}_{\tau}(X)]\). 

As per Section \ref{cvar}, the CVaR risk measure can be formulated as the continual learning objective \eqref{eq_cvar_1}, which is displayed below as Equation \eqref{eq_cvar_1_appendix} for convenience:
\begin{equation}
\label{eq_cvar_1_appendix}
\overline{\text{CVaR}}^{\pi_t} \doteq  \lim_{h \rightarrow{} \infty} \frac{1}{h - n_0} \sum_{n=n_0+1}^{h} \text{CVaR}_\tau[R_n \mid A_{n_0:n-1} \sim \pi^{\omega(n_0)}].
\end{equation}
In other words, we aimed to optimize the (left-tail) conditional value-at-risk associated with the limiting per-step reward distribution induced when following stationary policy \(\pi_t\). More specifically, our aim was to optimize the CVaR objective \eqref{eq_cvar_1_appendix} in two continual learning tasks via the \emph{RED CVaR Q-learning} algorithm proposed in \citet{Rojas2025-bf}. Importantly, the RED CVaR Q-learning algorithm was designed to optimize the CVaR associated with the long-run per-step reward distribution of an average-reward MDP, which precisely corresponds to the continual learning objective \eqref{eq_cvar_1_appendix}. The RED CVaR Q-learning algorithm (Algorithm \ref{alg_cvar_1}) is shown below:
\begin{algorithm}
   \caption{RED CVaR Q-Learning (Tabular) \citep{Rojas2025-bf}}
   \label{alg_cvar_1}
\begin{algorithmic}
    \STATE {\bfseries Input:} the policy \(\pi\) to be used (e.g., \(\varepsilon\)-greedy)
    \STATE {\bfseries Algorithm parameters:} step size parameters \(\alpha\), \(\alpha_{_{\text{CVaR}}}\), \(\alpha_{_{\text{VaR}}}\); CVaR parameter \(\tau\)
    \STATE Initialize \(Q(s, a) \: \forall s, a\) (e.g. to zero)
    \STATE Initialize CVaR arbitrarily (e.g. to zero)
    \STATE Initialize VaR arbitrarily (e.g. to zero)
    \STATE Obtain initial \(S\)
    \WHILE{still time to train}
        \STATE \(A \leftarrow\) action given by \(\pi\) for \(S\)
        \STATE Take action \(A\), observe \(R, S'\)
        \STATE \(\tilde{R} = \text{VaR} - \frac{1}{\tau} \max \{\text{VaR} - R, 0\}\)
        \STATE \(\delta = \tilde{R} - \text{CVaR} + \max_a Q(S', a) - Q(S, A)\)
        \STATE \(Q(S, A) = Q(S, A) + \alpha\delta\)
        \STATE \(\text{CVaR} = \text{CVaR} + \alpha_{_{\text{CVaR}}}\delta\)
        \IF{\(R \geq \text{VaR}\)}
            \STATE \(\text{VaR} = \text{VaR} + \alpha_{_{\text{VaR}}}(\delta + \text{CVaR} - \text{VaR})\)
        \ELSE
            \STATE \(\text{VaR} = \text{VaR} + \alpha_{_{\text{VaR}}}\left(\left(\frac{\tau}{ \tau - 1}\right)\delta + \text{CVaR} - \text{VaR}\right)\)
        \ENDIF
        \STATE \(S = S'\)
    \ENDWHILE
    \STATE return \(Q\)
\end{algorithmic}
\end{algorithm}

In terms of the two continual learning tasks considered in this work, we considered two continual variations of the \emph{red-pill blue-pill (RPBP)} task \citep{Rojas2025-bf}. More specifically, in the regular (non-continual) RPBP task, an agent, at each time step, can take either a \emph{red pill}, which takes them to the \emph{red world} state, or a \emph{blue pill}, which takes them to the \emph{blue world} state. Each state has its own characteristic per-step reward distribution, such that for a sufficiently low CVaR parameter, \(\tau\), the red world state has a reward distribution with a lower (worse) mean but a higher (better) CVaR compared to the blue world state. That is, in the regular RPBP task, for a sufficiently low CVaR parameter, \(\tau\), we would expect a risk-neutral agent to learn a policy that prefers to stay in the blue world, and a risk-averse agent to learn a policy that prefers to stay in the red world. 

We now discuss the two continual variations of the RPBP task considered in this work:

\subsection{\(\tau\)-RPBP Task}
In the first task, we considered a continual variation of RPBP, such that the \emph{risk attitude} of the agent, which is governed by the CVaR parameter, \(\tau\), changes over time from risk-neutral (\(\tau = 0.9 \)) to risk-averse (\(\tau = 0.1 \)). Here, we would expect the agent to first learn to stay in the blue world state, but change its preference to the red world state as its risk attitude changes from risk-neutral to risk-averse. More formally, this task can be viewed as a continual learning task, \(\{\mathcal{M}^k\}_{k=1}^2\), with a changing reward function, such that \(\mathcal{M}^k \doteq \langle \mathcal{S}, \mathcal{A}, \mathcal{R}_k, p \rangle\), where
\begin{equation}
\label{eq_cvar_2}
\tilde{R}_{t, k} = \text{VaR}_t - \frac{1}{\tau_{_k}}(\text{VaR}_t - R_t)^{+} \text{ (see Algorithm \ref{alg_cvar_1})},
\end{equation}
with \(\tau_{_1} = 0.9\) and \(\tau_{_2} = 0.1\). The indexing function, \(\omega\), was defined such that \(\omega(t) = 1\) for \(t < 50{,}000\), and \(\omega(t) = 2\) otherwise. That is, the agent's risk attitude changes from risk-neutral to risk-averse at \(t = 50{,}000\).

In terms of the hyperparameters used with the RED CVaR Q-learning algorithm, we used the tuned hyperparameters from \citet{Rojas2025-bf}. That is, \(\alpha = \text{2e-2}\), \(\alpha_{_{\text{CVaR}}} \doteq \eta_{_\text{CVaR}}\alpha\), where \(\eta_{_\text{CVaR}} = \text{1e-1}\), and \(\alpha_{_{\text{VaR}}} \doteq \eta_{_\text{VaR}}\alpha\), where \(\eta_{_\text{VaR}} = \text{1e-1}\). We used an \(\varepsilon\)-greedy policy with a fixed epsilon of 0.1, and set all initial guesses to zero. The reward distribution parameters used are the same as the parameters used in the \(k = 1\) portion of the \(\mathcal{S}\)-RPBP task discussed in Section \ref{s_rpbp} below. The results for this \emph{\(\tau\)-RPBP} task are shown in Figure \ref{fig_results_1}.

\subsection{\(\mathcal{S}\)-RPBP Task}
\label{s_rpbp}
In the second task, we considered another continual variation of RPBP. In this variation, the characteristic per-step reward distributions of the states change over time, such that the agent is required to continually adapt and find the state with the better CVaR (given a fixed risk attitude, \(\tau\)). More formally, this task can be viewed as a continual learning task with a changing state-space, such that a given state is effectively replaced with a state with a different per-step reward distribution. In other words, we have a continual learning task, \(\{\mathcal{M}^k\}\), such that \(\mathcal{M}^k \doteq \langle \mathcal{S}_k, \mathcal{A}, \mathcal{R}, p \rangle\). For a given \(\mathcal{S}_k\), the red world state reward distribution is characterized as a Gaussian distribution with mean \(\mu_{\text{red}}\) and standard deviation \(\sigma_{\text{red}}\). Conversely, the blue world state reward distribution is characterized as a mixture of two Gaussian distributions with means \(\mu_{\text{blue-a}}\) and \(\mu_{\text{blue-b}}\), standard deviations, \(\sigma_{\text{blue-a}}\) and \(\sigma_{\text{blue-b}}\), and a mixing coefficient of 0.5. 

In the experiment performed, we set \(k \in \{1, 2, 3\}\). For all \(k\) and all states, we set the standard deviation to 0.05. For \(k=1\), we set \(\mu_{\text{red}} = -0.7\), \(\mu_{\text{blue-a}} = -1.0\), and \(\mu_{\text{blue-b}} = -0.2\). For \(k=2\), we set \(\mu_{\text{red}} = -1.5\), \(\mu_{\text{blue-a}} = -1.25\), and \(\mu_{\text{blue-b}} = -1.0\). For \(k=3\), we set \(\mu_{\text{red}} = -0.5\), \(\mu_{\text{blue-a}} = -0.9\), and \(\mu_{\text{blue-b}} = -0.5\). The indexing function, \(\omega\), was defined such that \(\omega(t) = 1\) for \(t < 40{,}000\), \(\omega(t) = 2\) for \(40{,}000 \leq t < 80{,}000\), and \(\omega(t) = 3\) otherwise.

In terms of the hyperparameters used with the RED CVaR Q-learning algorithm, we used the tuned hyperparameters from \citet{Rojas2025-bf}. That is, \(\alpha = \text{2e-2}\), \(\alpha_{_{\text{CVaR}}} \doteq \eta_{_\text{CVaR}}\alpha\), where \(\eta_{_\text{CVaR}} = \text{1e-1}\), and \(\alpha_{_{\text{VaR}}} \doteq \eta_{_\text{VaR}}\alpha\), where \(\eta_{_\text{VaR}} = \text{1e-1}\). We used a fixed CVaR parameter, \(\tau\), of 0.25, an \(\varepsilon\)-greedy policy with a fixed epsilon of 0.1, and set all initial guesses to zero. The results for this \emph{\(\mathcal{S}\)-RPBP} task are shown in Figure \ref{fig_results_2}.

\subsection{Compute Time and Resources}
The experiments were conducted on an Apple M1 SoC (8-core CPU). For the \(\tau\)-RPBP experiment, each run took roughly 6 seconds to compute on a single CPU, for an approximate total of 5 CPU minutes across the 50 runs. For the \(\mathcal{S}\)-RPBP experiment, each run took roughly 4 seconds to compute on a single CPU, for an approximate total of 40 CPU seconds across the 10 runs.

\newpage

\section{Discounted MDPs and Plasticity}
\label{appendix_disc}
In this appendix, we provide a brief discussion that explores some of the challenges that discounted MDPs \citep{Puterman1994-dq} face when trying to maintain a non-zero level of risk-aware plasticity (as defined in Equation \eqref{eq_plasticity_risk}) in continual learning settings. More formally, a finite discounted MDP is the tuple \(\mathcal{M} \doteq \langle \mathcal{S}, \mathcal{A}, \mathcal{R}, p , \gamma\rangle\), where \(\mathcal{S}\) is a finite set of states, \(\mathcal{A}\) is a finite set of actions, \(\mathcal{R} \subset \mathbb{R}\) is a bounded set of rewards, \(p: \mathcal{S}\, \times\, \mathcal{A}\, \times\, \mathcal{R}\, \times\,  \mathcal{S} \rightarrow{} [0, 1]\) is a probabilistic transition function that describes the dynamics of the environment, and \(\gamma \in [0,1)\) is the discount factor. At each discrete time step, \(t = 0, 1, 2, \ldots\), an agent chooses an action, \(A_t \in \mathcal{A}\), based on its current state, \(S_t \in \mathcal{S}\), and receives a reward, \(R_{t+1} \in \mathcal{R}\), while transitioning to a (potentially) new state, \(S_{t+1}\), such that \(p(s', r \mid s, a) = \mathbb{P}(S_{t+1} = s', R_{t+1} = r \mid S_t = s, A_t = a)\). In a discounted MDP, an agent aims to find a policy, \(\pi: \mathcal{S} \rightarrow{} \mathcal{A}\), that optimizes the expected discounted sum of rewards. More formally, following the notation established in Section \ref{section_average_reward}, for a given (discounted) MDP, \(\mathcal{M}^k \in \{\mathcal{M}^k\}_{k=1}^\infty\), the performance of a stationary policy, \(\pi_t\), can be characterized at discrete time step \(N\) (where \(N > n_0 \geq 0\)) via the state-action value function, \(Q^{\pi_t}_{N}\), for all \(s \in \mathcal{S}_k\) and \(a \in \mathcal{A}_k\) as follows:
\begin{equation}
\label{eq_disc_value}
Q^{\pi_t}_{N}(s, a) \doteq \mathbb{E}\!\left[ \sum_{n=n_0}^{N-1} \gamma^{n-n_0} R_{n+1} \mid S_{n_0} = s, A_{n_0} = a, A_{n_0+1:N-1} \sim \pi^{\omega(n_0)} \right].
\end{equation}

Accordingly, as with the average-reward objective \eqref{eq_avg_reward}, we can consider the risk-aware analogue to the (risk-neutral) discounted objective \eqref{eq_disc_value}:
\begin{equation}
\label{eq_disc_value_risk}
\rho^{\pi_t}_{N}(s, a) \doteq \rho_{\tau}\!\left[ \sum_{n=n_0}^{N-1} \gamma^{n-n_0} R_{n+1} \mid S_{n_0} = s, A_{n_0} = a, A_{n_0+1:N-1} \sim \pi^{\omega(n_0)} \right].
\end{equation}

That is, the risk-aware objective \eqref{eq_disc_value_risk} aims to optimize some measure of risk pertaining to the distribution of discounted sum of rewards induced when following policy \(\pi_t\) from time step \(n_0\) up until time step \(N\). Here, \(\rho^{\pi_t}_{N}\) can be interpreted as a risk assessment (see Definition \ref{defn_risk_assessment}) performed at time step \(N\) via a static risk measure (see Definition \ref{defn_static}). In Proposition \ref{proposition_static_not_plastic}, we showed that static risk measures cannot maintain a non-zero level of risk-aware plasticity after time step \(N\). Building on this result, we will now show how utilizing the discounted MDP formulation can induce challenges when trying to maintain a non-zero level of risk-aware plasticity even before time step \(N\). 

In particular, consider Figure \ref{fig_discount_factors}, which evaluates the \(\gamma^\ell \doteq \gamma^{n - n_0}\) term in the risk-aware objective \eqref{eq_disc_value_risk} for \(\ell \geq 0\) across various discount factors, \(\gamma \in \{\text{0.5, 0.75, 0.9, 0.99, 0.999}\}\). As shown in the figure, the \(\gamma^\ell\) term evaluates to \(\approx 0\) for all the discount factors considered by \(\ell=6000\). Accordingly, rewards (i.e., observations) observed after \(\ell=6000\) effectively cease to affect the risk assessment. This implies that the agent is unable to maintain a non-zero level of risk-aware plasticity for all discount factors considered for any \(\ell \geq 6000\). As such, if the risk assessment is performed beyond this time step, then this implies that utilizing the discounted MDP formulation can induce plasticity challenges even before the (static) risk assessment is performed. Crucially, the average-reward MDP formulation does not utilize the \(\gamma^\ell\) term, and hence does not face such plasticity challenges.
\begin{figure}[htbp]
\centerline{\includegraphics[scale=0.51]{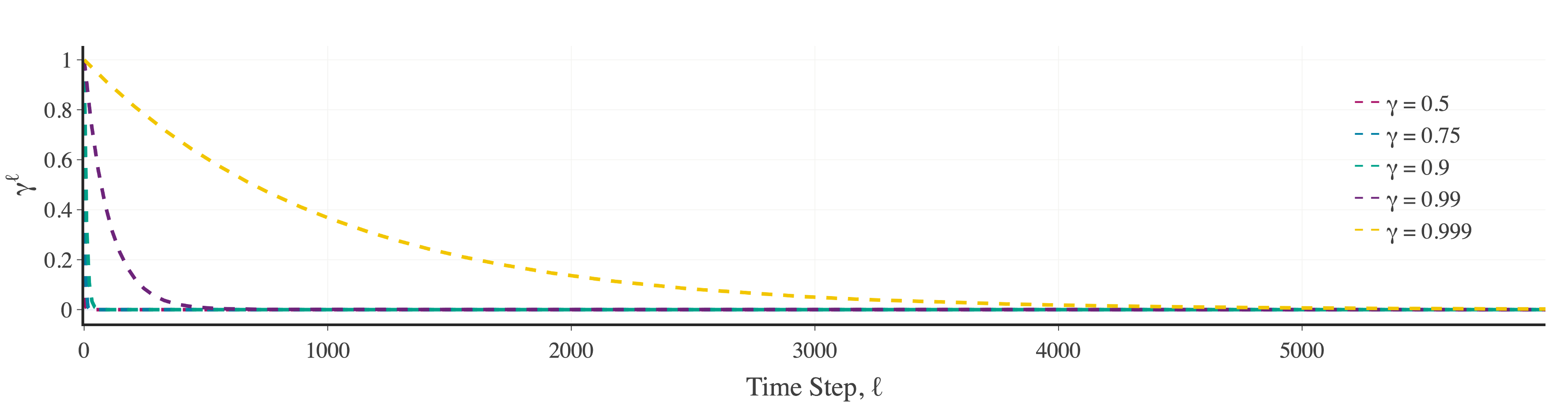}}
\caption{Evaluation of the \(\gamma^\ell \doteq \gamma^{n - n_0}\) term in the risk-aware objective \eqref{eq_disc_value_risk} for \(\ell \geq 0\) across various discount factors, \(\gamma \in \{\text{0.5, 0.75, 0.9, 0.99, 0.999}\}\). As shown in the figure, the \(\gamma^\ell\) term evaluates to \(\approx 0\) for all the discount factors considered by \(\ell=6000\), thereby eliminating the influence of rewards observed after \(\ell=6000\) on the risk assessment.}
\label{fig_discount_factors}
\end{figure}

\newpage

\section{Ergodic Risk Measures in Non-Continual Settings}
\label{appendix_noncontinual}
In this appendix, we provide a theoretical analysis that shows that in \emph{non-continual} settings, ergodic risk measures offer a new and appealing alternative for risk-aware learning and decision-making. First, we show that the risk-aware RL objective \eqref{eq_risk_rl} (which corresponds to an ergodic risk measure) can retain some notion of time consistency while remaining highly interpretable, thereby capturing the appeal of both static and nested risk measures. Then, we show that the risk-aware RL objective \eqref{eq_risk_rl} is able to circumvent some of the challenges and non-trivialities that have been historically associated with risk-aware RL objectives in non-continual settings.

For convenience, we simplify our notation for the risk-aware RL objective \eqref{eq_risk_rl} in non-continual settings as follows:
\begin{equation}
\label{eq_risk_rl_noncontinuous}
\bar{\rho}^{\pi}(s) \doteq \lim_{h \rightarrow{} \infty} \frac{1}{h} \sum_{n=1}^{h} \rho[R_n \mid S_{0}=s, A_{0:n-1} \sim \pi],
\end{equation}
such that it can be evaluated at time \(t\) as:
\begin{equation}
\label{eq_risk_rl_noncontinuous_eval}
\bar{\rho}^{\pi}_t(s) = \frac{1}{t} \sum_{n=1}^{t} \rho[R_n \mid S_{0}=s, A_{0:n-1} \sim \pi],
\end{equation}
where \(\pi\) denotes some policy and \(\rho[R_n | \cdot]\) corresponds to a risk functional which captures some distributional characteristic of the per-step rewards (e.g. variance, CVaR, etc.) with respect to some fixed risk attitude. Here, the non-continual risk-aware objective \eqref{eq_risk_rl_noncontinuous} can be viewed as being a specific instance of the more general risk-aware objective \eqref{eq_risk_rl}. Accordingly, given that we showed in Theorem \ref{theorem_ergodic_1} that the risk-aware objective \eqref{eq_risk_rl} corresponds to an ergodic risk measure, it directly follows that the non-continual risk-aware objective \eqref{eq_risk_rl_noncontinuous} also corresponds to an ergodic risk measure. This is formalized as Corollary \ref{corollary_ergodic} below:

\vspace{4pt}

\begin{corollary}
\label{corollary_ergodic}
Given an appropriate ergodicity-like assumption, such as Assumption \ref{assumption_unichain} or \ref{assumption_communicating}, and a stationary policy, \(\pi\), the non-continual risk-aware objective \eqref{eq_risk_rl_noncontinuous} corresponds to an ergodic risk measure, as defined in Definition \ref{defn_ergodic}.
\end{corollary}
\begin{proof}
This follows directly from Theorem \ref{theorem_ergodic_1}.
\end{proof}

For clarity, we note that as with the continual risk-aware objective \eqref{eq_risk_rl}, when we apply an appropriate ergodicity-like assumption, such as Assumption \ref{assumption_unichain} or \ref{assumption_communicating}, the risk-aware objective \eqref{eq_risk_rl_noncontinuous} becomes independent of prior conditions (i.e., \(\bar{\rho}^{\pi}(s) = \bar{\rho}^{\pi}\)). 

\subsection{Ergodic Risk Measures Retain Some Notion of Time Consistency}
In this section, we show that ergodic risk measures can retain some notion of time consistency while remaining highly interpretable, thereby capturing the appeal of both static and nested risk measures. In particular, we will show below that the risk-aware RL objective \eqref{eq_risk_rl_noncontinuous} satisfies a weaker notion of time consistency, which we call \emph{local time consistency}:

\vspace{4pt}

\begin{definition}[Local Time Consistency]  
\label{defn_local_time_consistency}  
Let \((\Omega, \mathcal{F}, \mathbb{P})\) denote a probability space with filtration \((\mathcal{F}_t)_{t=0}^\infty\), and let \(\{\rho_t\}_{t=0}^\infty\) denote a sequence of conditional risk measures, such that \(\rho_t : L^\infty(\mathcal{F}_\infty) \to L^\infty(\mathcal{F}_t)\). The sequence \(\{\rho_t\}_{t=0}^\infty\) is said to satisfy the \emph{local time consistency} property if there exists a time step, \(n \ge 0\), and a possibly infinite horizon length, \(m \in \mathbb{N}\cup\{\infty\}\), such that, for all \(n \le t < n + m\) and all \(X,X' \in L^\infty(\mathcal{F}_\infty)\), we have that: \(\rho_{t+1}(X) \le \rho_{t+1}(X') \ \implies \ \rho_t(X) \le \rho_t(X')\). That is, time consistency holds within some subset of the time-horizon. Note that when \(n = 0\) and \(m \to \infty\), this reduces to the standard definition of time consistency (Definition~\ref{defn_time_consistent}). 
\end{definition}

In essence, the above definition for local time consistency only requires that time consistency holds for some subset of the history, rather than the entire history. We note that such a notion of time consistency could be useful in a continual learning setting as it could provide some measure of \emph{stability}. That is, while we want the agent to have the flexibility to change its risk preferences over time, it would likely be problematic if the agent changed its risk preferences at every time step.

We will also make use of a relaxed variant of this property, which we refer to as \emph{\(\epsilon\)-local time consistency}, to account for the fact that in RL we are often dealing with approximated risk evaluations up to some arbitrary error threshold:

\newpage

\begin{definition}[\(\epsilon\)-Local Time Consistency]
\label{defn_eps_local_time_consistency}
Let \((\Omega, \mathcal{F}, \mathbb{P})\) denote a probability space with filtration \((\mathcal{F}_t)_{t=0}^\infty\), and let \(\{\rho_t\}_{t=0}^\infty\) denote a sequence of conditional risk measures, such that \(\rho_t : L^\infty(\mathcal{F}_\infty) \to L^\infty(\mathcal{F}_t)\). The sequence \(\{\rho_t\}_{t=0}^\infty\) is said to satisfy the \emph{\(\epsilon\)-local time consistency} property if, for a given \(\epsilon > 0\), there exists a time step, \(n \ge 0\), and a possibly infinite horizon length, \(m \in \mathbb{N}\cup\{\infty\}\), such that, for all \(n \le t < n + m\) and all \(X,X' \in L^\infty(\mathcal{F}_\infty)\), we have that: \(\rho_{t+1}(X) \le \rho_{t+1}(X') \ \implies \ \rho_t(X) \le \rho_t(X') + \epsilon\). That is, time consistency holds within some subset of the time-horizon up to an error threshold, \(\epsilon\).
\end{definition}

We will now show that, under ergodicity-like assumptions, the risk-aware RL objective \eqref{eq_risk_rl_noncontinuous} satisfies the \(\epsilon\)-local time consistency property, such that \(\epsilon \to 0\) as \(t \to \infty\):

\vspace{4pt}

\begin{proposition}
\label{proposition_local_time_consistency}
Given an appropriate ergodicity-like assumption, such as Assumption \ref{assumption_unichain} or \ref{assumption_communicating}, and a stationary policy, \(\pi\), the risk-aware objective \eqref{eq_risk_rl_noncontinuous} satisfies the \(\epsilon\)-local time consistency property (as defined in Definition \ref{defn_eps_local_time_consistency}), such that \(\epsilon \to 0\) as \(t \to \infty\).
\end{proposition}

\begin{proof}
Let \(\bar{\rho}^{\pi}_t\) denote the risk evaluated at time \(t\) (as per Equation \eqref{eq_risk_rl_noncontinuous}) under the stationary policy. Under the ergodicity-like assumption, we have, by Birkhoff's Ergodic Theorem \citep{Birkhoff1931-mx}, that \(\bar{\rho}^{\pi}_t\) converges to a stationary value as \(t \to \infty\). This implies that there exists a finite time step, \(j\), such that for all \(t \geq j\), \(\bar{\rho}^{\pi}_t\) converges to a stationary value. As such, for all \(t \geq j\), the evaluated risk at time \(t\) becomes arbitrarily indistinguishable from the evaluated risk at time \(t+1\), such that \(|\bar{\rho}^{\pi}_{t+1}(X) - \bar{\rho}^{\pi}_t(X)| \leq \tilde{\epsilon}\) with respect to some error threshold, \(\tilde{\epsilon}\), where \(\tilde{\epsilon} \to 0\) as \(t \to \infty\). Accordingly, for any future sequences, \(X\) and \(X'\), it directly follows that \(\bar{\rho}^{\pi}_{t+1}(X) \le \bar{\rho}^{\pi}_{t+1}(X') \implies \bar{\rho}^{\pi}_t(X) \le \bar{\rho}^{\pi}_t(X') + 2\tilde{\epsilon}\) for all \(t \geq j\), thereby satisfying the \(\epsilon\)-local time consistency property (where \(\epsilon = 2\tilde{\epsilon}\)), such that \(\tilde{\epsilon} \to 0\) as \(t \to \infty\). 
\end{proof}

As such, with Proposition \ref{proposition_local_time_consistency}, we have shown that our interpretable risk-aware objective \eqref{eq_risk_rl_noncontinuous}, satisfies some notion of time consistency, thereby capturing the appeal of both static (i.e., interpretability) and nested (i.e., time consistency) risk measures in the non-continual setting.

\subsection{Ergodic Risk Measures and Optimal Policies}
In this section, we show that the risk-aware RL objective \eqref{eq_risk_rl_noncontinuous} is able to circumvent some of the challenges and non-trivialities that have been historically associated with risk-aware RL objectives in non-continual settings. In particular, it is well-known that risk-aware RL objectives often admit non-Markovian optimal policies \citep{Bellemare2023-mn, Bowling2023-tm}. However, we will show below that the risk-aware objective \eqref{eq_risk_rl_noncontinuous}, which corresponds to an ergodic risk measure, provably admits a stationary, deterministic optimal policy, given the following assumption:

\vspace{4pt}

\begin{assumption}
\label{assumption_risk_functional_appendix}
Consider the risk functional in Equation \eqref{eq_risk_rl_noncontinuous}, \(\rho\), which captures some distributional characteristic of the limiting per-step reward distribution induced when following a given policy. We assume that \(\rho\) is bounded, and that, under an ergodicity-like assumption, the limiting per-step reward distribution induced by the policy contains all the information required to evaluate \(\rho\).
\end{assumption}

In essence, the above assumption ensures that the per-step reward distribution contains all the information needed to evaluate the risk functional, \(\rho\), at a given time step. In the context of the average-reward MDP formulation, we find this to be a reasonable assumption to make, given that in the risk-aware setting, we aim to find a policy that induces a limiting per-step reward distribution that is optimal with respect to some measure of risk. For example, if our risk functional corresponds to the variance of the distribution (such that our aim is to find a policy that induces a limiting per-step reward distribution with the best long-run variance), then Assumption \ref{assumption_risk_functional_appendix} implies that if we know this distribution, then we can compute its variance without needing additional information. 

Next, given Assumption \ref{assumption_risk_functional_appendix}, we can make use of Theorem 8.9.3 of \citet{Puterman1994-dq} to prove the following result:

\vspace{4pt}

\begin{theorem}
\label{theorem_c_1}
Consider the risk-aware objective \eqref{eq_risk_rl_noncontinuous}. Suppose that, under an appropriate ergodicity-like assumption, such as Assumption \ref{assumption_unichain} or \ref{assumption_communicating}, this objective exists and is unique. Furthermore, suppose that Assumption \ref{assumption_risk_functional_appendix} holds, and that there exists some policy that is optimal with respect to the risk-aware objective \eqref{eq_risk_rl_noncontinuous}. Then, there exists a stationary, deterministic policy that is optimal with respect to the risk-aware objective \eqref{eq_risk_rl_noncontinuous}.
\end{theorem}
\begin{proof}
Consider the risk functional from Equation \eqref{eq_risk_rl_noncontinuous} under the ergodicity-like assumption: \(\rho[R_n \mid A_{0:n-1} \sim \pi]\). Given Assumption \ref{assumption_risk_functional_appendix}, we know that this risk functional can be computed entirely based on information from the per-step rewards, \(R_t\). Consequently, given that the per-step rewards are computed based on the current states and actions (i.e., \(S_t\) and \(A_t\)), the risk functional can be viewed as the following mapping: \(\mathcal{S} \times \mathcal{A} \to \mathbb{R}\), where \(\mathcal{S}\) and \(\mathcal{A}\) denote the state and action-spaces of an average-reward MDP. Accordingly, we can interpret \(\rho\) as a (``risk-aware'') reward function, in the usual MDP sense, such that the risk-aware objective can be written as follows:
\begin{subequations}
\label{eq_c_1}
\begin{align}
\label{eq_c_1_1}
\bar{\rho}^{\pi} &\doteq  \lim_{h \rightarrow{} \infty} \frac{1}{h} \sum_{n=1}^{h} \rho[R_n \mid A_{0:n-1} \sim \pi]\\
\label{eq_c_1_2}
&= \sum_{s \in \mathcal{S}}\mu_{\pi}(s)\sum_{a \in \mathcal{A}}\pi(a | s)\hat{r}(s,a),
\end{align}
\end{subequations}
where \(\hat{r}(s,a)\) denotes the risk-aware reward function (for some \(s \in \mathcal{S}, a \in \mathcal{A}\)) that corresponds to the risk functional, \(\rho\), and \(\mu_{\pi}\) denotes the limiting distribution of states induced by policy \(\pi\) (we know that \(\mu_{\pi}\) exists and is unique given the ergodicity-like assumption).

As such, we can see from Equation \eqref{eq_c_1} that optimizing our risk-aware objective, \(\bar{\rho}^{\pi}\), amounts to optimizing a standard average-reward MDP objective with a modified (i.e., ``risk-aware'') reward function that satisfies the requirements of reward functions for average-reward MDPs. Consequently, it follows from Theorem 8.9.3 of \citet{Puterman1994-dq} that if there exists some policy that is optimal with respect to \(\bar{\rho}^{\pi}\), then, there exists a stationary, deterministic policy that is optimal with respect to \(\bar{\rho}^{\pi}\). We note that Theorem 8.9.3 of \citet{Puterman1994-dq} assumes finite state and action-spaces, as well as bounded rewards. It can be easily shown that these assumptions are satisfied by Assumption \ref{assumption_risk_functional_appendix} and by the MDP formulation considered in this work (see Section \ref{section_average_reward}). This completes the proof.
\end{proof}

In essence, Theorem \ref{theorem_c_1} states that if an optimal policy exists for the risk-aware objective \eqref{eq_risk_rl_noncontinuous}, then there exists a stationary, deterministic policy that is also optimal with respect to this objective. Intuitively, this follows from Equation \eqref{eq_c_1_2}, which shows that the objective depends neither on time nor the history, thereby making a stationary, deterministic policy sufficient for optimality.

Finally, we can leverage Theorems 8.4.5 or 9.1.8 of \citet{Puterman1994-dq} to show that an optimal policy indeed exists:

\vspace{4pt}

\begin{theorem}
\label{theorem_c_2}
Consider the risk-aware objective \eqref{eq_risk_rl_noncontinuous}. Suppose that, under an appropriate ergodicity-like assumption, such as Assumption \ref{assumption_unichain} or \ref{assumption_communicating}, this objective exists and is unique. Furthermore, suppose that Assumption \ref{assumption_risk_functional_appendix} holds. Then, there exists a policy that is optimal with respect to the risk-aware objective \eqref{eq_risk_rl_noncontinuous}.
\end{theorem}
\begin{proof}
We showed in the proof for Theorem \ref{theorem_c_1} how optimizing our risk-aware objective, \(\bar{\rho}^{\pi}\), ultimately amounts to optimizing a standard average-reward MDP objective with a modified reward function that satisfies the requirements of reward functions for average-reward MDPs. Consequently, it directly follows from Theorems 8.4.5 or 9.1.8 of \citet{Puterman1994-dq} (depending on which ergodicity-like assumption is used) that there exists some policy that is optimal with respect to \(\bar{\rho}^{\pi}\). We note that these theorems from \citet{Puterman1994-dq} assume finite state and action-spaces, as well as bounded rewards. It can be easily shown that these assumptions are satisfied by Assumption \ref{assumption_risk_functional_appendix} and by the MDP formulation considered in this work (see Section \ref{section_average_reward}). This completes the proof. 
\end{proof}

As such, with Theorems \ref{theorem_c_1} and \ref{theorem_c_2}, we have shown that the risk-aware objective \eqref{eq_risk_rl_noncontinuous}, which corresponds to an ergodic risk measure, provably admits a stationary, deterministic optimal policy, thereby circumventing some of the challenges and non-trivialities that have been historically associated with risk-aware RL objectives in non-continual settings. We conclude this section with a remark related to the interpretation of the risk-aware reward function, \(\hat{r}(s,a)\):

\vspace{4pt}

\begin{remark}
At first glance, the interpretation of the risk functional, \(\rho\), as a reward function, \(\hat{r}(s,a)\), in Equation \eqref{eq_c_1} may give the impression that we are ``averaging away'' the risk from our objective. However, this is far from reality. In particular, Equation \eqref{eq_c_1} can be thought of as computing the ``time-averaged risk'' associated with the limiting per-step reward distribution (e.g. the average long-run variance or the average long-run CVaR of \(R_t\)). Accordingly, the objective \eqref{eq_risk_rl_noncontinuous} precisely captures some (time-averaged) distributional characteristic associated with the limiting per-step reward distribution beyond the mean, thereby making it a risk-aware objective. 
\end{remark}

\newpage

\section{Ergodicity and Continual Learning}
\label{appendix_ergodicity}
In this appendix, we provide a mathematically rigorous discussion on the relationship between ergodicity and continual learning, in which we show that ergodicity-like assumptions, such as Assumptions \ref{assumption_unichain} or \ref{assumption_communicating}, do not preclude an agent's ability to be a continual learner. Given that this discussion is risk-agnostic (i.e., it equally applies to risk-neutral and risk-aware settings), we will limit our discussion to risk-neutral agents; however, the same arguments apply to risk-aware agents.

To make our arguments, we will leverage the language and formalism introduced in \citet{Abel2023-xe}, in which an instance of RL is considered to be an instance of continual RL (CRL) if, informally, the best agent never stops learning. To this end, we now provide key definitions from \citet{Abel2023-xe}:

\vspace{2pt}

\begin{definition}[Definition 2.1 of \citet{Abel2023-xe}]
\label{defn_interface}
An \emph{agent-environment interface} is the pair \((\mathcal{A}, \mathcal{O})\), such that \(|\mathcal{A}| \geq 2\) and \(|\mathcal{O}| \geq 1\), where \(\mathcal{O}\) and \(\mathcal{A}\) denote observation and action-spaces.
\end{definition}

\vspace{2pt}

\begin{definition}[Definition 2.2 of \citet{Abel2023-xe}]
\label{defn_histories}
The \emph{histories} with respect to interface \((\mathcal{A}, \mathcal{O})\) are the set of sequences of action-observation pairs,
\begin{equation}
    \mathcal{H} = \bigcup_{t=0}^{\infty}(\mathcal{A} \times \mathcal{O})^t.
\end{equation}
\end{definition}

\vspace{2pt}

\begin{definition}[Definition 2.3 of \citet{Abel2023-xe}]
\label{defn_environment}
An \emph{environment} with respect to interface \((\mathcal{A}, \mathcal{O})\) is a function, \(e: \mathcal{H} \times \mathcal{A} \to \Delta(\mathcal{O})\), where \(\Delta(\mathcal{O})\) denotes the probability simplex over the set of observations, \(\mathcal{O}\).
\end{definition}

\vspace{2pt}

\begin{definition}[Definition 2.4 of \citet{Abel2023-xe}]
\label{defn_agent}
An \emph{agent} with respect to interface \((\mathcal{A}, \mathcal{O})\) is a function, \(\lambda: \mathcal{H} \to \Delta(\mathcal{A})\), where \(\Delta(\mathcal{A})\) denotes the probability simplex over the set of actions, \(\mathcal{A}\).
\end{definition}

\vspace{2pt}

\begin{definition}[Definition 2.5 of \citet{Abel2023-xe}]
\label{defn_realizable_histories}
The \emph{realizable histories} of a given agent-environment pair, \((\lambda, e)\), define the set of histories of any length that can occur with non-zero probability from the interaction of \(\lambda\) and \(e\), such that:
\begin{equation}
    \mathcal{H}^{\lambda,e} = \overline{\mathcal{H}} = \bigcup_{t=0}^{\infty} \left\{ h_t \in \mathcal{H}_t : \prod_{k=0}^{t-1} e(o_{k+1} \mid h_k, a_k)\lambda(a_k \mid h_k) > 0 \right\}.
\end{equation}
\end{definition}

\vspace{2pt}

\vspace{2pt}

\begin{definition}[Definition 2.8 of \citet{Abel2023-xe}]
\label{defn_performance}
Let \(\Lambda\) denote a set of agents and \(\mathcal{E}\) denote a set of environments. The \emph{performance}, \(\upsilon: \mathcal{H} \times \Lambda \times \mathcal{E} \to [v_{min}, v_{max}]\), is a bounded function for fixed constants, \(v_{min}, v_{max} \in \mathbb{R}\).
\end{definition}

In words, \(\mathcal{H}\) denotes the set of all possible finite interaction histories (with \(\overline{\mathcal{H}}\) denoting the subset of realizable histories), where a history, \(h_t \in \mathcal{H}\), corresponds to the sequence of observations and actions up to time \(t\) (this is analogous to the sequences \(O_{1:t}\) and \(A_{1:t}\), as defined in Section \ref{prelim_crl} of this work). An agent is defined as a mapping, \(\lambda: \mathcal{H} \to \Delta(\mathcal{A})\), that outputs a probability distribution over actions given a history. Following \citet{Abel2023-xe}, we let \(\Lambda\) denote the set of all such agents considered in a given RL problem. The performance function, \(\upsilon(\lambda, e)\), expresses some statistic of the future rewards produced by the interaction between an agent, \(\lambda\), and an environment, \(e\). We note that the definition of performance (i.e., Definition \ref{defn_performance}) can be viewed as a risk-neutral instance of the definition of a risk assessment (see Definition \ref{defn_risk_assessment}) proposed in this work.

To formalize the learning process, \citet{Abel2023-xe} utilize the concept of an \emph{agent basis}, \(\Lambda_B \subset \Lambda\), which represents a core set of fixed, baseline behaviours (e.g. a finite set of stationary policies or fixed neural network weights) that an agent implicitly searches over as it interacts with the environment. To dictate how an agent switches between these base behaviours, \citet{Abel2023-xe} introduce the concept of a learning rule, which gives rise to the \emph{generates} operator, as follows:

\vspace{2pt}

\begin{definition}[Definition 3.2 of \citet{Abel2023-xe}]
\label{defn_learning_rule}
A \emph{learning rule} over an agent basis, \(\Lambda_B\), is a function, \(\sigma: \mathcal{H} \to \Lambda_B\), that selects a base agent for each history.
\end{definition}

\vspace{2pt}

\begin{definition}[Definition 3.4 of \citet{Abel2023-xe}]
\label{defn_generates}
We say a basis \(\Lambda_B\) \emph{generates} \(\Lambda\) in \(e\), denoted as \(\Lambda_B \models \Lambda\), if and only if there exists a set of learning rules, \(\Sigma\), such that:
\begin{equation}
    \forall_{\lambda \in \Lambda} \exists_{\sigma \in \Sigma} \forall_{h \in \overline{\mathcal{H}}} \lambda(h) = \sigma(h)(h).
\end{equation}
\end{definition}

Intuitively, a basis, \(\Lambda_B\), generates an agent set, \(\Lambda\), if every agent in \(\Lambda\) can be understood as simply
switching between the base agents (i.e., the agents in \(\Lambda_B\)) according to some learning rule.

Given this search process, an agent is said to ``stop learning'' if it eventually adopts one of these fixed base behaviours permanently. In particular, an agent, \(\lambda\), is said to \emph{reach} \(\Lambda_B\) in environment \(e\), denoted as \(\lambda \leadsto \Lambda_B\), where (\(\leadsto\)) denotes the \emph{reaches} operator, if there exists some realizable history after which \(\lambda\) becomes equivalent to some base agent, \(\lambda_B \in \Lambda_B\), for all future histories. By contrast, an agent that never permanently settles on a single base strategy is said to never reach the basis, such that \(\lambda \not\leadsto \Lambda_B\). Through this formalism, \citet{Abel2023-xe} propose the following definitions for a continual learning agent and continual RL:

\vspace{2pt}

\begin{definition}[Definition 4.1 of \citet{Abel2023-xe}]
\label{defn_continual_agent}
An agent, \(\lambda\), is a \emph{continual learning agent} in \(e\) relative to \(\Lambda_B\) if and only if the basis generates the agent (such that \(\Lambda_B \models \{\lambda\}\)) and the agent never reaches the basis (such that \(\lambda \not\leadsto \Lambda_B\)).
\end{definition}

\vspace{2pt}

\begin{definition}[Definition 4.2 of \citet{Abel2023-xe}]
\label{defn_continual_rl}
Consider an RL problem, \((e, \upsilon, \Lambda)\). Let \(\Lambda_B \subset \Lambda\) be a basis such that \(\Lambda_B \models \Lambda\) and let \(\Lambda^* = \arg\max_{\lambda \in \Lambda} \upsilon(\lambda, e)\). We say \((e, \upsilon, \Lambda, \Lambda_B)\) defines a \emph{continual RL} problem if:
\begin{equation}
\forall_{\lambda^* \in \Lambda^*} \lambda^* \not\leadsto \Lambda_B.    
\end{equation}
\end{definition}

Informally, Definition \ref{defn_continual_rl} states that an instance of RL is considered to be an instance of continual RL if the best agent never stops learning.

As such, having formalized the notion of continual RL and a continual learning agent, we are now ready to begin our discussion on the role of ergodicity in continual RL. In particular, we will show that ergodicity-like assumptions, such as Assumptions \ref{assumption_unichain} or \ref{assumption_communicating}, do not preclude an agent's ability to be a continual learner. To make this argument, we will seek to answer the following question:

\textit{Suppose that there exists some continual RL (CRL) problem, \((e, \upsilon, \Lambda, \Lambda_B)\), such that it satisfies Definition \ref{defn_continual_rl}. Can we determine whether the environment component of this CRL problem, \(e\), satisfies an ergodicity-like assumption, based solely on the fact that we know that the overarching problem satisfies the definition of a CRL problem?} 

In other words, we want to determine whether ergodicity-like assumptions are strictly incompatible with continual RL. That is, if we know that the problem is an instance of CRL, does it automatically invalidate the possibility that the environment is ergodic (and vice-versa)? In Theorem \ref{theorem_crl_and_ergodicity} below, we answer this question, and show that ergodicity neither implies nor precludes continual learning:

\vspace{4pt}

\begin{theorem}
\label{theorem_crl_and_ergodicity}
Let \((e, \upsilon, \Lambda, \Lambda_B)\) denote a continual RL problem, such that it satisfies Definition \ref{defn_continual_rl}. The fact that this continual RL problem satisfies Definition \ref{defn_continual_rl} neither implies nor precludes the condition that the environment, \(e\), satisfies an ergodicity-like assumption.
\end{theorem}
\begin{proof}
We observe that the environment must either satisfy an ergodicity-like assumption, or not satisfy an ergodicity-like assumption. Accordingly, we proceed with a proof by contradiction for both mutually exclusive cases:

\textbf{Case 1:} Suppose, for the sake of contradiction, that every valid continual RL problem strictly requires a non-ergodic environment. Now consider an environment with a single state and two available actions, \(a_1\) and \(a_2\). It can be easily shown that this environment satisfies standard ergodicity-like assumptions. Now, let the reward function change periodically, such that \(a_1\) yields the optimal reward for the first \(n\) steps, \(a_2\) yields the optimal reward for the next \(m\) steps, and so forth, with the pattern repeating indefinitely, such that an optimal agent must continually adapt its policy to track the shifting rewards. Because the agent never permanently settles on a single stationary base behaviour (i.e., \(\lambda \not\leadsto \Lambda_B\)), this constitutes a valid continual RL problem within an ergodic environment, thereby contradicting the assumption that CRL strictly implies a non-ergodic environment.

\textbf{Case 2:} Suppose, for the sake of contradiction, that every valid continual RL problem strictly requires an ergodic environment. Now consider an environment consisting of two isolated, absorbing states, each with two available actions, such that switching between the two states is structurally impossible. It can be easily shown that this environment violates any standard ergodicity-like assumption. At time \(t=0\), the agent is placed into one of these states at random. Within each state, the rewards associated with each available action continually change over time. Accordingly, because the optimal action continues to shift indefinitely, the agent must continually adapt its policy and never reaches a fixed basis (such that \(\lambda \not\leadsto \Lambda_B\)). As such, this constitutes a valid continual RL problem within a non-ergodic environment, thereby contradicting the assumption that CRL strictly implies an ergodic environment.

Hence, given that both cases result in a contradiction, we can conclude that the satisfaction of Definition \ref{defn_continual_rl} neither implies nor precludes the condition that the environment, \(e\), satisfies an ergodicity-like assumption. This concludes the proof.
\end{proof}

As such, with Theorem \ref{theorem_crl_and_ergodicity}, we have shown that ergodicity and continual learning are not strictly incompatible. We note that although the counterexamples used in the proof for Theorem \ref{theorem_crl_and_ergodicity} both make use of non-stationary reward functions, there are many possible counterexamples that could be considered. For instance, one could consider an agent that has a highly constrained function approximator (such as a linear function approximator) operating in a large, stationary environment. Because the agent cannot represent the entire environment at once, it must continually update its representations as it explores different regions of the environment, thereby constituting an instance of continual RL (as per Definition \ref{defn_continual_rl}). Importantly, whether or not the environment satisfies ergodicity-like assumptions would not impact whether this example constitutes an instance of continual RL, given that the agent would be required to update its representations both in (large, stationary) ergodic and (large, stationary) non-ergodic environments.

\newpage



\end{document}